\documentclass{article}

\usepackage[final]{neurips_2019}

\usepackage[utf8]{inputenc} 
\usepackage[T1]{fontenc}    
\usepackage{hyperref}       
\usepackage{url}            
\usepackage{booktabs}       
\usepackage{amsfonts}       
\usepackage{nicefrac}       
\usepackage{microtype}      
\usepackage{color}

\usepackage{algorithm,algorithmicx,algpseudocode} 
\usepackage{mathrsfs,amsmath,amssymb,amsthm}
\usepackage{multirow}
\usepackage{dsfont}
\usepackage{thmtools,thm-restate}
\usepackage{graphicx, caption, subcaption, wrapfig}

\newcommand{\sign}{sign}

\usepackage{natbib}
\setcitestyle{square,comma,numbers,compress}
\usepackage[toc]{appendix}

\newtheorem{theorem}{Theorem}
\newtheorem{lemma}[theorem]{Lemma}

\newtheorem{corollary}[theorem]{Corollary}
\newtheorem*{problem*}{Problem}
\newtheorem{claim}[theorem]{Claim}
\newtheorem{example}{Example}
\newtheorem{definition}{Definition}
\newtheorem{fact}[theorem]{Fact}

\title{On Sample Complexity Upper and Lower Bounds for Exact Ranking from Noisy Comparisons}

\author{%
	Wenbo Ren\\
	Dept. Computer Science \& Engineering\\
	The Ohio State University\\
	\texttt{ren.453@osu.edu} \\
	\And 
	Jia Liu \\
	Dept. Electrical \& Computer Engineering\\
	The Ohio State University\\
	\texttt{liu.1736@osu.edu} \\
	\AND
	Ness B. Shroff \\
	Dept. Computer Science \& Engineering and Electrical \& Computer Engineering \\
	The Ohio State University\\
	\texttt{shroff.11@osu.edu} \\
}

\begin{document}

\maketitle

\begin{abstract}%
	This paper studies the problem of finding the exact ranking from noisy comparisons. A comparison over a set of $m$ items produces a noisy outcome about the most preferred item, and reveals some information about the ranking. By repeatedly and adaptively choosing items to compare, we want to fully rank the items with a certain confidence, and use as few comparisons as possible. Different from most previous works, in this paper, we have three main novelties: (i) compared to prior works, our upper bounds (algorithms) and lower bounds on the sample complexity (aka number of comparisons) require the \textit{minimal} assumptions on the instances, and are not restricted to specific models; (ii) we give lower bounds and upper bounds on instances with \textit{unequal} noise levels; and (iii) this paper aims at the \textit{exact} ranking without knowledge on the instances, while most of the previous works either focus on approximate rankings or study exact ranking but require prior knowledge. We first derive lower bounds for pairwise ranking (i.e., compare two items each time), and then propose (nearly) \textit{optimal} pairwise ranking algorithms. We further make extensions to listwise ranking (i.e., comparing multiple items each time). Numerical results also show our improvements against the state of the art.
\end{abstract}

\section{Introduction}\label{Sec:Intro}

\subsection{Background and motivation}
Ranking from noisy comparisons has been a canonical problem in the machine learning community, and has found applications in various areas such as social choices \citep{SocialChoice2005Conitzer}, web search \citep{WebSearch2001}, crowd sourcing \citep{CrowdSourcing2013}, and recommendation systems \citep{RecommendationSystem2010}. The main goal of ranking problems is to recover the full or partial rankings of a set of items from noisy comparisons. The items can refer to various things, such as products, movies, pages, and advertisements, and the comparisons refer to tests or queries about the items' strengths or the users' preferences. In this paper, we use words ``item'', ``comparison'' and ``preference'' for simplicity. A comparison involves two (i.e., pairwise) or multiple (i.e., listwise) items, and returns a noisy result about the most preferred one, where ``noisy'' means that the comparison outcome is random and the returned item may not be the most preferred one. 
A noisy comparison reveals some information about the ranking of the items. This information can be used to describe users' preferences, which helps applications such as recommendations, decision making, and advertising, etc. One example is e-commerce: A user's click or purchase of a product (but not others) is based on a noisy (due to the lack of full information) comparison between several similar products, and one can rank the products based on the noisy outcomes of the clicks or the purchases to give better recommendations. Due to the wide applications, in this paper, we do not focus on specific applications and regard comparisons as black-box procedures.

This paper studies the \textit{active} (or adaptive) ranking, where the learner adaptively chooses items to compare based on previous comparison results, and returns a ranking when having enough confidence. Previous works \citep{CrowdSourcing2013,AdaptivePooling2012} have shown that, compared to non-adaptive ranking, active ranking can significantly reduce the number of comparisons needed and achieve a similar confidence or accuracy. In some applications such as news apps, the servers are able to adaptively choose news to present to the users and collect feedbacks, by which they can learn the users' preferences in shorter time compared to non-adaptive methods and may provide better user experience.

We focus on the active {\em full} ranking problem, that is, to find the \textit{exact} full ranking with a certain confidence level by adaptively choosing the items to compare, and try to use as few comparisons as possible. The comparisons can be either pairwise (i.e., comparing two items each time) or listwise (i.e., comparing more than two items each time). We are interested in the upper and lower bounds on the sample complexity (aka number of comparisons needed). We are also interested in understanding whether using listwise comparisons can reduce the sample complexity.

\subsection{Models and problem statement}
There are $n$ items in total, indexed by $1,2,3,...,n$. Given a comparison over a set $S$, each item $i\in S$ has $p_{i,S}$ probability to be returned as the most preferred one (also referred to as $i$ ``wins'' this comparison), and when a tie happens, we randomly assign one item as the winner, which makes $\sum_{i\in S}p_{i,S}=1$ for all set $S\subset [n]$. When $|S|=2$, we say this comparison is pairwise, and when $|S|>2$, we say listwise. In this paper, a comparison is said to be $m$-wise if it involves exactly $m$ items (i.e., $|S|=m$). For $m=2$ and a two-sized set $S = \{i,j\}$, to simplify notation, we define $p_{i,j}:=p_{i,S}$ and $p_{j,i}:=p_{j,S}$.

\textit{Assumptions.} In this paper, we make the following assumptions: \textbf{A1)} Comparisons are independent across items, sets, and time. We note that the assumption of independence is common in the this area (e.g., \cite{MaxingAndRanking2017,RankingLimits2018,Falahatgar2017,ActiveRanking2019,ApproximateRanking2018,CoarseRanking2018,Subsetwise2019,ListwisePLMaxing2019,Simple2017,OnlineRankingElicitation2015}).
\textbf{A2)} There is a unique permutation $(r_1,r_2,...,r_n)$ of $[n]$ \footnote{For any positive integer $k$, define $[k] := \{1,2,...,k\}$ to simplify notation} such that $r_1\succ\!r_2\succ\!\cdots \succ\!r_n$, where $i\succ\!j$ denotes that $i$ ranks higher than $j$ (i.e., $i$ is more preferred than $j$). We refer to this unique permutation as the \textit{true ranking} or \textit{exact ranking}, and our goal is to recover the true ranking.
\textbf{A3)} For any set $S$ and item $i\in S$, if $i$ ranks higher than all other items $k$ of $S$, then $p_{i,S}>p_{k,S}$. For pairwise comparisons, A3 states that $i\succ j$ if and only if $p_{i,j} > 1/2$. We note that for pairwise comparisons, A3 can be viewed as the weak stochastic transitivity \citep{TransitivityModel2016}. The three assumptions are necessary to make the exact ranking (i.e., finding the unique true ranking) problem meaningful, and thus, we say our assumptions are minimal. Except for the above three assumptions, we do \textit{not} assume any prior knowledge of the $p_{i,S}$ values. We note that any comparison model can be fully described by the comparison probabilities $(p_{i,S}:i\in S, S\subset [n])$.

We further define some notations. Two items $i$ and $j$ are said to be \textit{adjacent} if in the true ranking, there does \textit{not} exist an item $k$ such that $i\succ\!k\succ\!j$ or $j\succ\!k\succ\!i$. For all items $i$ and $j$ in $[n]$, define $\Delta_{i,j}:=|p_{i,j}-1/2|$, $\Delta_i:=\min_{j\neq i}\Delta_{i,j}$, and $\tilde{\Delta}_i := \min\{\Delta_{i,j}:i\mbox{ and } j\mbox{ are adjacent}\}$. For any $a,b\in\mathbb{R}$, define $a\land b := \min\{a,b\}$ and $a\lor b := \max\{a,b\}$. 

We adopt the notion of strong stochastic transitivity (SST) \cite{RankingLimits2018}: for all items $i$, $j$, and $k$ satisfying $i\succ\!j\succ\!k$, it holds that $p_{i,k}\geq \max\{p_{i,j},p_{j,k}\}$.
Under the SST condition, we have $\Delta_i = \tilde{\Delta}_i$ for all items $i$. We note that this paper is not restricted to the SST condition. Pairwise (listwise) ranking refers to ranking from pairwise (listwise) comparisons.

In this paper, $f \preceq g$ means $f = O(g)$, $f \succeq g$ means $f = \Omega(g)$, and $f \simeq g$ means $f = \Theta(g)$. The meanings of $O(\cdot)$, $\Omega(\cdot)$, and $\Theta(\cdot)$ are standard in the sense of Bachmann-Landau notation. In this paper, we define another notation $\tilde{\Omega}(\cdot)$ which is similar to $\Omega(\cdot)$ but has weaker requirements. We state the definition of $\tilde{\Omega}(\cdot)$ in Definition~\ref{Def:TildeOmega}. This definition is inspired by \cite{ChenBestArmIdentification2015}, and when $k = 1$, it is of the same form as the formula in \cite[Theorem~D.1]{ChenBestArmIdentification2015}. We use $f \tilde{\succeq} g$ to denote $f = \tilde{\Omega}(g)$ for simplicity. 

\begin{definition}[Definition of $\tilde{\Omega}(\cdot)$]\label{Def:TildeOmega}
    Let $k$ be a positive integer and define $E_i := [e^i, e^{i+1})$ for any positive integer $i$. Two function $f(\mathbf{x})$ and $g(\mathbf{x})$ ($\mathbf{x}\in\mathbb{R}^k$) are said to satisfy $f(\mathbf{x}) = \tilde{\Omega}(g(\mathbf{x}))$ if there is a constant $c_0$ such that for any constant $\gamma > 0$ we have
    \begin{align}
        \limsup_{N\rightarrow\infty} \frac{\sum_{(i_1, i_2,..., i_k)\in [N]^k}\mathds{1}\{\exists \mathbf{x} \in E_{i_1}\times E_{i_2}\times \cdots \times E_{i_k}: f(\mathbf{x}) < c_0 g(\mathbf{x}) \}}{N^{k-1 + \gamma}} = 0.
    \end{align}
\end{definition}

In other words, the notation $\tilde{\Omega}(\cdot)$ means that except a negligible proportion of the intervals (or cells), we have $f(\mathbf{x}) \geq c_0g(\mathbf{x})$, and in this case we can say that $f(\mathbf{x})$ is ``almost'' $\Omega(g(\mathbf{x}))$. All the above asymptotic notations are with respect to $n$, $\delta^{-1}$, $\epsilon^{-1}$, $\Delta^{-1}$, $\eta^{-1}$, $(\Delta^{-1}_{i},i\in [n])$, and $(\tilde{\Delta}^{-1}_{i},i\in [n])$. 


\begin{problem*}[Exact ranking]\label{Def:ExactRanking}
	Given $\delta\in(0,1/2)$ and $n$ items, one wants to determine the true ranking with probability at least $1-\delta$ by adaptively choosing sets of items to compare.
\end{problem*}

\begin{definition}[$\delta$-correct algorithms]\label{Def:AlgCorrectness}
	An algorithm is said to be $\delta$-correct for a problem if for any input instance of this problem, it, with probability at least $1-\delta$, returns a correct result in finite time.
\end{definition}

\subsection{Main results} 
First, for $\delta$-correct pairwise ranking algorithms with no prior knowledge of the instances, we derive a  lower bound of the form $\Omega(\sum_{i\in[n]}\Delta_i^{-2}(\log\log\Delta_i^{-1}+\log(n/\delta)))$ \footnote{All $\log$ in this paper, unless explicitly noted, are natural $\log$.}, which is shown to be \textit{tight} (up to constant factors) under SST and some mild conditions. 

Second, for pairwise and listwise ranking under the multinomial logit (MNL) model, we derive a model-specific lower bound, which is \textit{tight} (up to constant factors) under some mild conditions, and shows that in the worst case, the listwise lower bound is no lower than the pairwise one.

Third, we propose a pairwise ranking algorithm that requires no prior information and minimal assumptions on the instances, and its sample-complexity upper bound {\em matches} the lower bounds proved in this paper under the SST condition and some mild conditions, implying that both upper and lower bounds are optimal.

\section{Related works}\label{Sec:RW}
Dating back to 1994, the authors of \cite{NoisyComputing1994} studied the noisy ranking under the strict constraint that $p_{i,j} \geq 1/2 +\Delta$ for any $i\succ\!j$, where $\Delta >0$ is priorly \textit{known}. They showed that any $\delta$-correct algorithm needs $\Theta({n}{\Delta^{-2}}\log(n/\delta))$ comparisons for the worst instances. However, in some cases, it is impossible to either assume the knowledge of $\Delta$ or require $p_{i,j}\geq 1/2+\Delta$ for any $i\succ\! j$. Also, their bounds only depend on the minimal gap $\Delta$ but not $\Delta_{i,j}$'s or $\Delta_i$'s, and hence is not tight in most cases. 
In contrast, our algorithms require \textit{no knowledge} on the gaps (i.e., $\Delta_{i,j}$'s), and we establish sample-complexity lower bounds and upper bounds that base on \textit{unequal} gaps, which can be much tighter when $\Delta_i$'s vary a lot.

Another line of research is to explore the {\em probably approximately correct} (PAC) ranking (which aims at finding a permutation $(r_1,r_2,...,r_n)$ of $[n]$ such that $p_{r_i,r_j} \geq 1/2-\epsilon$ for all $i < j$, where $\epsilon>0$ is a given error tolerance) under various pairwise comparison models \citep{MaxingAndRanking2017,RankingLimits2018,Falahatgar2017,RankingBounds2018,Subsetwise2019,ListwisePLMaxing2019,OnlineRankingElicitation2015}. 
When $\epsilon >0$, the PAC ranking may not be unique.  The authors of \cite{MaxingAndRanking2017,RankingLimits2018,Falahatgar2017} proposed algorithms with $O({n}{\epsilon^{-2}}\log(n/\delta))$ upper bound for PAC ranking with tolerance $\epsilon>0$ under SST and the stochastic triangle inequality\footnote{Stochastic triangle inequality means that for all items $i,j,k$ with $i\succ\!j\succ\!k$, $\Delta_{i,k} \leq \Delta_{i,j}+\Delta_{j,k}$.} (STI).
When $\epsilon$ goes to zero, the PAC ranking reduces to the true ranking. However, when $\epsilon>0$, we still need some prior knowledge on $(p_{i,j}:i,j\in [n])$ to get the true ranking, as we need to know a lower bound of the values of $\Delta_{i,j}$ to ensure that the PAC ranking equals to the unique true ranking. When $\epsilon = 0$, the algorithms in \cite{MaxingAndRanking2017,RankingLimits2018,Falahatgar2017} do not work. Prior to these works, the authors of \cite{OnlineRankingElicitation2015} also studied the PAC ranking. In their work, with $\epsilon = 0$, the unique true ranking can be found by $O(n\log{n}\cdot\max_{i\in[n]}\{\frac{1}{\Delta_i^2}\log(\frac{n}{\delta\Delta_i}\})$ comparisons, which is higher than the lower bound and upper bound proved in this paper by at least a log factor.

In contrast, this paper is focused on recovering the {\em unique} true (exact) ranking, and there are three major motivations. First, in some applications, we prefer to find the exact order, especially in ``winner-takes-all'' situations. For example, when predicting the winner of an election, we prefer to get the exact result but not the PAC one, as only a few votes can completely change the result.
Second, analyzing the exact ranking can help us better understand the instance-wise upper and lower bounds about the ranking problems, while the bounds of PAC ranking (e.g., in \cite{MaxingAndRanking2017,RankingLimits2018,Falahatgar2017}) may only work for the worst cases. Third, exact ranking algorithms may better exploit the large gaps (e.g., $\Delta_i$'s) to achieve lower sample complexities. In fact, when finding the PAC ranking, we can perform the exact ranking algorithm and the PAC ranking algorithm parallelly, and return a ranking whenever one of them returns. By this, when $\epsilon$ is large, we can benefit from the PAC upper bounds that depend on $\epsilon^{-2}$, and when $\epsilon$ is small, we can benefit from the exact ranking bounds that depend on $\Delta_i^{-2}$.

There are also other interesting active ranking works. The authors of \cite{ActiveRanking2019,ApproximateRanking2018,CoarseRanking2018,Simple2017} studied active ranking under the Borda-Score model, where the Borda-Score of item $i$ is defined as $\frac{1}{n-1}\sum_{j\neq i}p_{i,j}$. We note that the Borda-Score model does not satisfy A2 and A3 and is not comparable with the model in this paper. There are also many works on best item(s) selection, including \cite{LimitedRounds2017,MNLlistwise2018,SpectralMLE2015,ListwisePL2017,Mohajer2016active,RankCentrarity2016,ListwisePLMaxing2019}, which are less related to this paper.

\section{Lower bound analysis}\label{Sec:LB}

\subsection{Generic lower bound for $\delta$-correct algorithms}\label{Sec:LBdA}

In this subsection, we establish a sample-complexity lower bound for pairwise ranking. The lower bound is for $\delta$-correct algorithms, which have performance guarantee for all input instances. There are algorithms that work faster than our lower bound but only return correct results with $1-\delta$ confidence for a restricted class of instances, which is discussed in Section~\ref{Sec:FRDiscussions}. Theorem~\ref{Theorem:LB} states the lower bound, and its full proof is provided in Section~\ref{Sec:FRProofs}. Here we remind that $\tilde{\Delta}_{i}:= \min\{\Delta_{i,j}:i\mbox{ and } j \mbox{ are adjaent}\}$.

\begin{restatable}[Lower bound for pairwise ranking]{theorem}{RestateLB}
	\label{Theorem:LB}
	\footnote{In the previous version, the $\Omega(\cdot)$ notation we were using is not accurate, and we now use a more accurate notation $\tilde{\Omega}(\cdot)$. We thank Björn Haddenhorst at Paderborn University, Germany for bringing this issue to our attention.}Given $\delta\in(0,1/12)$ and an instance $\mathcal{I}$ with $n$ items, then the number of comparisons used by a $\delta$-correct algorithm $\mathcal{A}$ with no prior knowledge about the gaps of $\mathcal{I}$ is lower bounded by
	\begin{align}\label{Eq:GLB}
	    \tilde{\Omega}\Big(\sum_{i\in[n]}\frac{1}{\tilde{\Delta}_i^{2}}\log\log\frac{1}{\tilde{\Delta}_i}\Big) + \Omega\Big(\sum_{i\in[n]}\frac{1}{\tilde{\Delta}_i^{2}}\log\frac{1}{\delta} + \min\Big\{\sum_{i\in[n]}\frac{1}{\tilde{\Delta}_i^{2}}\log\frac{1}{x_i}: \sum_{i\in[n]}x_i \leq 1\Big\}\Big).
	\end{align}
	If $\delta \preceq 1/poly(n)$\footnote{$poly(n)$ means a polynomial function of $n$, and $\delta \preceq 1/poly(n)$ means $\delta \preceq n^{-p}$ for some constant $p>0$.}, or $\max_{i,j\in[n]}\{\tilde{\Delta}_i/\tilde{\Delta}_j\} \preceq n^{1/2-p}$ for some constant $p>0$, then the lower bound becomes
	\begin{align}\label{Eq:LB}
	\tilde{\Omega}\big(\sum_{i\in[n]}\tilde{\Delta}_i^{-2}\log\log\tilde{\Delta}_i^{-1}\big) + \Omega \big(\sum_{i\in[n]}\tilde{\Delta}_i^{-2}\log(n/\delta)\big).
	\end{align}
\end{restatable}

\textbf{Remark:} (i) When the instance satisfies the SST condition (the algorithm does not need to know this information), the bound in Eq.~(\ref{Eq:LB}) is tight (up to a constant factor) under the given condition, which will be shown in Theorem~\ref{Theorem:TP-IIR} later. (ii) The lower bound in Eq.~(\ref{Eq:GLB}) implies an $n\log{n}$ term in $\min\{\cdot\}$, which can be checked by the convexity of $\log(1/x_i)$ and Jensen's inequality, which yields $\sum_{i\in[n]}\log(1/x_i)\geq n\log(n/\sum_{i\in[n]}x_i) \geq n\log{n}$. (iii) The lower bound in (\ref{Eq:LB}) may not hold if the required conditions do not hold, which will be discussed in Section~\ref{Sec:FRDiscussions}.

\begin{proof}[Proof sketch of Theorem~\ref{Theorem:LB}.]
	We outline the basic idea of the proof here and refer readers to Section~\ref{Sec:FRProofs} for details.
	Our first step is to use the results in \cite{RatioTest1964,LIL2014,FKLowerBound2004} to establish a lower bound for ranking two items. 
	Then, it seems straightforward that the lower bound for ranking $n$ items can be obtained by summing up the lower bounds for ranking $\{q_1,q_2\}$, $\{q_2,q_3\}$,...,$\{q_{n-1},q_n\}$, where $q_1\succ\!q_2\succ\!\cdots\succ\!q_n$ is the true ranking. 
	However, Note that to rank $q_i$ and $q_j$, there may be an algorithm that compares $q_i$ and $q_j$ with other items like $q_k$, and uses the comparison outcomes over $\{q_i,q_k\}$ and $\{q_j,q_k\}$ to determine the order of $q_i$ and $q_j$.
	Since it is unclear to what degree comparing $q_i$ and $q_j$ with other items can help to rank $q_i$ and $q_j$, the lower bound for ranking $n$ items cannot be simply obtained by summing up the lower bounds for ranking 2 items.
	To overcome this challenge, our strategy is to construct two problems: $\mathcal{P}_1$ and $\mathcal{P}_2$ with decreasing influence of this type of comparisons. Then, we prove that $\mathcal{P}_1$ reduces to exact ranking and $\mathcal{P}_2$ reduces to $\mathcal{P}_1$. Third, we prove a lower bound on $\delta$-correct algorithms for solving $\mathcal{P}_2$, which yields a lower bound for exact ranking. Finally, we use this lower bound to get the desired lower bounds in Eq.~(\ref{Eq:GLB}) and Eq.~(\ref{Eq:LB}).
\end{proof}

\subsection{Model-specific lower bound}
In Section~\ref{Sec:LBdA}, we provide a lower bound for $\delta$-correct algorithms that do not require any knowledge of the instances except assumptions A1 to A3. However, in some applications, people may focus on a specific model, and hence, the algorithm may have further knowledge about the instances, such as the model's restrictions.
Hence, the lower bound in Theorem~\ref{Theorem:LB} may not be applicable any more\footnote{For example, under a model with $\Delta_{i,j}=\Delta$ for any $i\neq j$ where $\Delta>0$ is unknown, one may first estimate a lower bound of $\Delta$, and then perform algorithms in \cite{NoisyComputing1994}, yielding a sample complexity lower than Theorem~\ref{Theorem:LB}.}.

In this paper, we derive a model-specific lower bound for the MNL model. The MNL model can be applied to both pairwise and listwise comparisons. For pairwise comparisons, the MNL model is mathematically equivalent to the Bradley-Terry-Luce (BTL) model \citep{Luce2012} and the Plackett-Luce (PL) model \citep{OnlineRankingElicitation2015}. There have been many prior works that focus on active ranking based on this model (e.g., \cite{MNLlistwise2018,BothOptimal2017,SpectralMLE2015,ActiveRanking2019,ListwisePL2017,RankCentrarity2016,Subsetwise2019,OnlineRankingElicitation2015}). 

Under the MNL model, each item holds a real number representing the users' preference over this item, where the larger the number, the more preferred the item. Specifically, each item $i$ holds a parameter $\gamma_i\in \mathbb{R}$ such that for any set $S$ containing $i$, $p_{i,S}=\exp(\gamma_i)/\sum_{j\in S}\exp(\gamma_j)$. To simplify notation, we let $\theta_i=\exp(\gamma_i)$, hence, $p_{i,S}=\theta_i/\sum_{j\in S}\theta_j$. We name $\theta_i$ as the \textit{preference score} of item $i$. We define $\Delta_{i,j}:= |p_{i,j}-1/2|$, $\Delta_i := \min_{j\neq i}\Delta_{i,j}$, and we have $\tilde{\Delta}_i = \Delta_i$, i.e., the MNL model satisfies the SST condition.

\begin{restatable}{theorem}{RestateLBMNL}[MNL Lower Bound]\label{Theorem:LBMNL}
	Let $\delta\in(0,1/12)$ and given a $\delta$-correct algorithm $\mathcal{A}$ with the knowledge that the input instances satisfy the MNL model, let $N_\mathcal{A}$ be the number of comparisons conducted by $\mathcal{A}$, then $\mathbb{E}[N_\mathcal{A}]$ is lower bounded by Eq. (\ref{Eq:GLB}) with a (possibly) different hidden constant factor. When $\delta \preceq 1/poly(n)$ or $\max_{i,j\in[n]}\{\Delta_i/\Delta_j\} \preceq n^{1/2-p}$ for some constant $p>0$, the sample complexity is lower bounded by Eq. (\ref{Eq:LB}) with a (possibly) different hidden constant factor.
\end{restatable}
\begin{proof}[Proof sketch.]
	We prove this theorem by Lemmas~\ref{Lemma:HeadsLowerBound}, \ref{Lemma:LBCoins} and \ref{Lemma:MNLReduction}, which could be of independent interest.
	
	Suppose that there are two coins with unknown \textit{head probabilities} (the probability that a toss produces a head) $\lambda$ and $\mu$, respectively, and we want to find the more biased one (i.e., the one with the larger head probability). Lemma~\ref{Lemma:HeadsLowerBound} states a lower bound on the number of heads or tails generated for finding the more biased coin, which works even if $\lambda$ and $\mu$ go to $0$. 
	This is in contrast to the lower bounds on the number of tosses given by previous works \citep{LIL2014,LowerBound2012,FKLowerBound2004}, which go to infinity as $\lambda$ and $\mu$ go to 0.
	
	\begin{restatable}[Lower bound on number of heads]{lemma}{RestateHeadsLowerBound}\label{Lemma:HeadsLowerBound}
		Let $\lambda+\mu\leq 1$, $\Delta:=|{\lambda}/({\lambda+\mu})-1/2|$, and $\delta\in(0,1/2)$ be given. To find the more biased coin with probability $1-\delta$, any $\delta$-correct algorithm for this problem produces $\tilde{\Omega}(\Delta^{-2}\log\log{\Delta^{-1}}) + \Omega({\Delta^{-2}}\log{\delta^{-1}})$ heads in expectation. 
	\end{restatable}
	
	Now we consider $n$ coins $C_1,C_2,...,C_n$ with mean rewards $\mu_1,\mu_2,...,\mu_n$, respectively, where for any $i\in[n]$, $\theta_i/\mu_i=c$ for some constant $c>0$. Define the gaps of coins $\Delta^c_{i,j} := |\mu_i/(\mu_i + \mu_j)-1/2|$, and $\Delta^c_i := \min_{j\neq i}\Delta^c_{i,j}$. We can check that for all $i$ and $j$, $\Delta^c_{i,j} = \Delta_{i,j}$, and $\Delta_i = \tilde{\Delta}_i = \Delta^c_{i}$. 
	
	\begin{restatable}[Lower bound for arranging coins]{lemma}{RestateLBCoins}\label{Lemma:LBCoins}
		For $\delta<1/12$, to arrange these coins in ascending order of head probabilities, the number of heads generated by any $\delta$-correct algorithm is lower bounded by  Eq.~(\ref{Eq:GLB}) with a (possibly) different hidden constant factor.
	\end{restatable}
	
	The next lemma shows that any algorithm that solves a ranking problem under the MNL model can be transformed to solve the pure exploration multi-armed bandit (PEMAB) problem with Bernoulli rewards(e.g., \cite{LIL2014,Halving2010,QuantileBandit2019}). Previous works \cite{LimitedRounds2017,ActiveRanking2019,ApproximateRanking2018} have shown that certain types of pairwise ranking problems (e.g., Borda-Score ranking) can also be transformed to PEMAB problems.
	But in this paper, we make a {\em reverse connection} that bridges these two classes of problems, which may be of independent interest. We note that in our prior work \cite{RankingBounds2018}, we proved a similar result.
	
	\begin{restatable}[Reducing PEMAB problems to ranking]{lemma}{RestateMNLReduction} \label{Lemma:MNLReduction}
		If there is a $\delta$-correct algorithm that correctly ranks $[n]$ with probability $1-\delta$ by $M$ expected number of comparisons, then we can construct another $\delta$-correct algorithm that correctly arranges the coins $C_1,C_2,...,C_n$ in the order of ascending head probabilities with probability $1-\delta$ and produces $M$ heads in expectation.
	\end{restatable}
	
	The theorem follows by Lemmas~\ref{Lemma:LBCoins} and \ref{Lemma:MNLReduction}. A full proof can be found in Section~\ref{Sec:FRProofs}.
\end{proof}

\subsection{Discussions on listwise ranking}
A listwise comparison compares $m$ ($m>2$) items and returns a noisy result about the most preferred item. It is an interesting question whether exact ranking from listwise comparisons requires less comparisons. The answer is ``It depends.'' 
When every comparison returns the most preferred item with high probability (w.h.p.)\footnote{In this paper, ``w.h.p.'' means with probability at least $1-n^{-p}$, where $p > 0$ is a sufficiently large constant.}, then, by conducting $m$-wise comparisons, the number of comparisons needed for exact ranking is $\Theta(n\log_m{n})$, i.e., there is a $\log{m}$ reduction, which is stated in Proposition~\ref{Proposition:HPC}. The proof can be found in Section~\ref{Sec:FRProofs}.

\begin{restatable}[Listwise ranking with negligible noises]{proposition}{RestatePropositionHPC}\label{Proposition:HPC}
	If all comparisons are correct w.h.p., to exactly rank $n$ items w.h.p. by using $m$-wise comparisons, $\Theta(n\log_m{n})$ comparisons are needed.
\end{restatable}

In general, when the ``w.h.p. condition'' is violated, listwise ranking does not necessarily require less comparisons than pairwise ranking (in order sense). Here, we give an example. For more general models, it remains an open problem to identify the theoretical limits, which is left for future studies.

\begin{restatable}{theorem}{RestateLBLWMNL}\label{Theorem:LBLWMNL}
	Under the MNL model, given $n$ items with preference scores $\theta_1,\theta_2,...,\theta_n$ and $\Delta_{i,j} := |\theta_i/(\theta_i+\theta_j)-1/2|$, $\tilde{\Delta}_i=\Delta_i:=\min_{j\neq i}\Delta_{i,j}$, to correctly rank these $n$ items with probability $1-\delta$, even with $m$-wise comparisons for all $m\in\{2,3,...,n\}$, the lower bound is the same as the pairwise ranking (i.e., Theorem~\ref{Theorem:LBMNL}) with (possibly) different hidden constant factors.
\end{restatable}

Theorem~\ref{Theorem:LBLWMNL} gives a minimax lower bound for listwise ranking, which is the same as pairwise ranking. The proof is given in Section~\ref{Sec:FRProofs}. The authors of \cite{MNLlistwise2018} have shown that for top-$k$ item selection under the MNL model, listwise comparisons can reduce the number of comparisons needed compared with pairwise comparisons. However, for exact ranking, listwise comparisons cannot.

\section{Algorithms and the upper bound for pairwise ranking}\label{Sec:Alg}
In this section, we establish a (nearly) sample-complexity optimal $\delta$-correct algorithm for exact ranking, where whether the word ``nearly'' can be deleted depends on the structures of the instances. The algorithm is based on Binary Search proposed in \cite{NoisyComputing1994} with upper bound $O(n\Delta_{\min}^{-2}\log(n/\delta))$, where $\Delta_{\min}:=\min_{i\neq j}\Delta_{i,j}$. Binary Search has two limitations: (i) it requires the knowledge of $\Delta_{\min}$ {\em a priori} to run, and (ii) it does not utilize the unequal noise levels. 

In this paper, we propose a technique named \textit{Attempting with error prevention} and establish a corresponding insertion subroutine that attempts to insert an item $i$ into a sorted list with a guessing $\Delta_i$-value, while preventing errors from happening if the guess is not well chosen. If the guess is small enough, this subroutine correctly inserts the item with a large probability, and if not, this subroutine will, with a large probability, not insert the item into a wrong position. By attempting to insert item $i$ with diminishing guesses of $\Delta_i$, this subroutine finally correctly inserts item $i$ with a large confidence.

\algtext*{EndIf}
\makeatletter
\renewcommand{\ALG@name}{Subroutine}
\makeatother

\begin{algorithm}[h]
	\caption{Attempting-Comparison$(i,j,\epsilon, \delta)$ (ATC)}\label{Sbrt:ATC}
	\textbf{Initialize:} $\forall t$, let $b^t=\sqrt{\frac{1}{2t}\log\frac{\pi^2t^2}{3\delta}}$; $b^{max}\gets\lceil\frac{1}{2\epsilon^2}\log\frac{2}{\delta}\rceil$; $w_i\gets 0$;
	\begin{algorithmic}[1]
		\For{$t \gets 1$ to $b^{max}$}
		\State Compare $i$ and $j$ once; Update $w_i\gets w_i+1$ if $i$ wins; Update $\hat{p}^t_i\gets w_i/t$;
		\If{$\hat{p}^t_i > 1/2+b^t$}
		\Return{$i$};
		\EndIf
		\If{$\hat{p}^t_i < 1/2-b^t$}
		\Return{$j$};
		\EndIf
		\EndFor
		\State \Return{$i$ \textbf{if} $\hat{p}^t_i>1/2$;} 
		\Return{$j$ \textbf{if} $\hat{p}^t_i<1/2$};
		and \Return{a random item \textbf{if} $\hat{p}^t_i=1/2$};
	\end{algorithmic}
\end{algorithm}

To implement the technique ``Attempting with error prevention'', we first need to construct a useful subroutine called Attempting-Comparison (ATC), which attempts to rank two items with $\epsilon$, a guess of $\Delta_{i,j}$. Then, by ATC, we establish Attempting-Insertion (ATI), which also adopts this technique.

\begin{restatable}[Theoretical Performance of ATC]{lemma}{RestateTPATC}\label{Lemma:TP-ATC}
	ATC terminates after at most $b^{max}=O({\epsilon^{-2}}\log{(1/\delta)})$ comparisons and returns the more preferred item with probability at least $1/2$. Further, if $\epsilon \leq \Delta_{i,j}$, then ATC returns the more preferred item with probability at least $1-\delta$.
\end{restatable}	

Next, to establish insertion subroutine ATI, we introduce preference interval trees \cite{NoisyComputing1994} (PIT). A PIT is constructed from a sorted list of items. For a sorted list of items $S$ with size $l$, without loss of generality, we assume that $r_1\succ r_2\succ\cdots\succ r_l$. We introduce two artificial items $-\infty$ and $+\infty$, where $-\infty$ is such that $p_{i,-\infty}=1$ for any item $i$, and $+\infty$ is such that $p_{i,+\infty}=0$ for any item $i$. 

\begin{wrapfigure}{r}{5cm}
	\centering
	\includegraphics[width=0.32\textwidth]{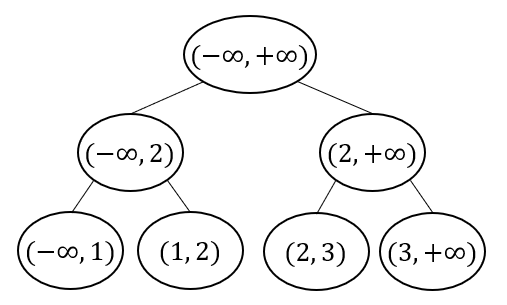}
	\caption{An example of PIT, constructed from a sorted list with three items $3\succ 2\succ 1$.}\label{fig:PITSample}
\end{wrapfigure}
\textbf{Preference Interval Tree }\cite{NoisyComputing1994}.
A preference interval tree constructed from the sorted list $S$ satisfies the following conditions: (i) It is a binary tree with depth $\lceil 1+\log_2(|S|+1)\rceil$. (ii) Each node $u$ holds an interval $(u\mbox{.left},u\mbox{.right})$ where $u\mbox{.left},u\mbox{.right}\in S\cup\{-\infty,+\infty\}$, and if $u$ is non-leaf, it holds an item $u\mbox{.mid}$ satisfying $u\mbox{.right}\succ u\mbox{.mid}\succ u\mbox{.left}$. (iii) A node $i$ is in the interval $(j,k)$ if and only if $k\succ i\succ j$. (iv) The root node is with interval $(-\infty,+\infty)$. From left to right, the leaf nodes are with intervals $(-\infty, r_l), (r_l,r_{l-1}),(r_{l-1},r_{l-2}),...,(r_{2},r_1),(r_1,+\infty)$. (v) Each non-leaf node $u$ has two children $u\mbox{.lchild}$ and $u\mbox{.rchild}$ such that $u\mbox{.left}=u\mbox{.lchild}\mbox{.left}$, $u\mbox{.right}=u\mbox{.rchild}\mbox{.right}$ and $u\mbox{.mid}=u\mbox{.lchild}\mbox{.right}=u\mbox{.rchild}\mbox{.left}$.

\begin{algorithm}[h]
	\caption{Attempting-Insertion$(i,S,\epsilon, \delta)$ (ATI).}\label{Sbrt:ATI}
	\textbf{Initialize:}
	Let $T$ be a PIT constructed from $S$;  $h\gets\lceil 1 + \log_2(1+|S|)\rceil$, the depth of $T$; \\
	For all leaf nodes $u$ of $T$, initialize $c_u\gets 0$; Set $t^{\max}\gets \lceil\max\{4h, \frac{512}{25}\log\frac{2}{\delta}\}\rceil$ and $q\gets \frac{15}{16}$;
	\begin{algorithmic}[1]
		\State $X\gets$ the root node of $T$;
		\For{$t \gets$ $1$ to $t^{\max}$}
		\If{$X$ is the root node}
		\If{ATC$(i,X\mbox{.mid},\epsilon,1-q)$ = $i$} 
		$X\gets X\mbox{.rchild}$;\quad \quad\#i.e., ATC returns $i\succ X.\mbox{mid}$
		\Else\ $X\gets X\mbox{.lchild}$;
		\EndIf
		\ElsIf{$X$ is a leaf node}{
			\If{ATC$(i,X\mbox{.left},\epsilon,1-\sqrt{q})=i$ $\land$ ATC$(i,X\mbox{.right},\epsilon,1-\sqrt{q})= X\mbox{.right}$}
			\State $c_X\gets c_X+1$;
			\If{$c_X> b^t := \frac{1}{2}t + \sqrt{\frac{t}{2}\log\frac{\pi^2 t^2}{3\delta}} + 1$}
			\State Insert $i$ into the corresponding interval of $X$ and \Return{\textit{inserted}};
			\EndIf
			\ElsIf{$c_X>0$} $c_X\gets c_X-1$
			\Else\ $X\gets X.\mbox{parent}$
			\EndIf
		}
		\Else
		\If{ATC$(i,X\mbox{.left},\epsilon,1-\sqrt[3]{q})=X\mbox{.left}$ $\lor$ ATC$(i,X\mbox{.right},\epsilon,1-\sqrt[3]{q})=i$}
		\State $X\gets X.\mbox{parent}$;
		\ElsIf{ATC$(i,X\mbox{.mid},\epsilon,1-\sqrt[3]{q})=i$}
		$X\gets X\mbox{.rchild}$;
		\Else\ $X\gets X\mbox{.lchild}$;
		\EndIf
		\EndIf
		\EndFor
		\If{there is a leaf node $u$ with $c_u\geq 1+\frac{5}{16}t^{\max}$}
		\State Insert $i$ into the corresponding interval of $u$ and \Return{\textit{inserted}};
		\Else\ \Return \textit{unsure};
		\EndIf
	\end{algorithmic}
\end{algorithm}

Based on the notion of PIT, we present insertion subroutine ATI in Subroutine~\ref{Sbrt:ATI}. ATI runs a random walk on the PIT to insert $i$ into $S$. Let $X$ be the point that moves on the tree.
We say a leaf $u_0$ \textit{correct} if the item $i$ belongs to $(u.\mbox{left},u.\mbox{right})$. Define $d(X):=$ the distance (i.e., the number of edges) between $X$ and $u_0$. At each round of the subroutine, if all comparisons give correct results, we say this round is \textit{correct}, otherwise we say \textit{incorrect}. For each correct round, either $d(X)$ is decreased by 1 or the counter of $u_0$ is increased by 1. The subroutine inserts $i$ into $u_0$ if $u_0$ is counted for $1+\frac{5}{16}t^{\max}$ times. Thus, after $t^{\max}$ rounds, the subroutine correctly inserts $i$ into $S$ if the number of correct rounds is no less than $\frac{21}{32}t^{\max}+\frac{h}{2}$, where $h = \lceil 1 + \log_2(|S|+1)\rceil$ is the depth of the tree. If guessing $\epsilon\leq \Delta_i$, then each round is correct with probability at least $q$, making the subroutine correctly insert item $i$ with probability at least $1-\delta$.

For all $\epsilon>0$, each round is incorrect with probability at most $1/2$, and thus, by concentration inequalities, we can also show that with probability at least $1-\delta$, $i$ will not be placed into any leaf node other than $u_0$. That is, if $\epsilon > \Delta_i$, the subroutine either correctly inserts $i$ or returns \textit{unsure} with probability at least $1-\delta$. The choice of parameters guarantees the sample complexity. Lemma~\ref{Lemma:TP-ATI} states its theoretical performance, and the proof is relegated to the supplementary material.

\begin{restatable}[Theoretical performance of ATI]{lemma}{RestateTPATI}\label{Lemma:TP-ATI}
	Let $\delta \in (0,1)$. ATI returns after $O({\epsilon^{-2}}\log({|S|}/{\delta}))$ comparisons and, with probability at least $1-\delta$, correctly inserts $i$ or returns \textit{unsure}. Further, if $\epsilon \leq \Delta_{i}$, it correctly inserts $i$ with probability at least $1-\delta$.
\end{restatable}

By Lemma~\ref{Lemma:TP-ATI}, we can see that the idea ``Attempting with error prevention'' is successfully implemented. Thus, by repeatedly attempting to insert an item with diminishing guess $\epsilon$ with proper confidences for the attempts, one can finally correctly insert $i$ with probability $1-\delta$. We use this idea to establish the insertion subroutine Iterative-Attempting-Insertion (IAI), and then use it to establish the ranking algorithm Iterative-Insertion-Ranking (IIR). Their theoretical performances are stated in Lemma~\ref{Lemma:TP-IAI} and Theorem~\ref{Theorem:TP-IIR}, respectively, and their proofs are given in supplementary material.

\begin{minipage}{0.51\textwidth}
	\begin{algorithm}[H]
		\caption{Iterative-Attempting-Insertion (IAI).}\label{Sbrt:IAI}
		\textbf{Input parameters:} $(i,S,\delta)$;\\ 
		\textbf{Initialize:} For all $\tau\in\mathbb{Z}^+$, set $\epsilon_\tau={2^{-(\tau+1)}}$ and $\delta_\tau=\frac{6\delta}{\pi^2 \tau^2}$; $t\gets 0$; $Flag\gets$ \textit{unsure};
		\begin{algorithmic}[1]
			\Repeat\ \ $t\gets t+1$;
			\State $Flag\gets$ATI$(i,S,\epsilon_t,\delta_t)$;
			\Until{$Flag=$}\textit{ inserted} 
		\end{algorithmic}
	\end{algorithm}
\end{minipage}\ 
\makeatletter
\renewcommand{\ALG@name}{Algorithm}
\makeatother
\begin{minipage}{0.48\textwidth}
	\begin{algorithm}[H]
		\caption{Iterative-Insertion-Ranking (IIR).}\label{Alg:IIR}
		\textbf{Input:} $S = [n]$, and confidence $\delta>0$;
		\begin{algorithmic}[1]
			\State $Ans\gets $ the list containing only $S[1]$;
			\For{$t\gets$ $2$ to $|S|$}
			\State IAI$(S[t],Ans,\delta/(n-1))$;
			\EndFor
			\State \Return{$Ans$}; 
		\end{algorithmic}
	\end{algorithm}
\end{minipage}

\begin{restatable}[Theoretical Performance of IAI]{lemma}{RestateTPIAI}\label{Lemma:TP-IAI}
	With probability at least $1-\delta$, IAI correctly inserts $i$ into $S$, and conducts at most $O({\Delta_i^{-2}}(\log\log{\Delta_i^{-1}}+\log({|S|}/{\delta})))$ comparisons.
\end{restatable}

\algtext{EndFor}{\textbf{end for}}

\begin{restatable}[Theoretical Performance of IIR]{theorem}{RestateTPIIR}\label{Theorem:TP-IIR}
	With probability at least $1-\delta$, IIR returns the exact ranking of $[n]$, and conducts at most $O(\sum_{i\in[n]}{{\Delta_i^{-2}}(\log\log{\Delta_i^{-1}}+\log({n}/\delta))})$ comparisons.
\end{restatable}

\textbf{Remark:} We can see that the upper bounds of IIR depend on the values of $(\Delta_i,i\in[n])$ while the lower bounds given in Theorem~\ref{Theorem:LB} depend on the values of $(\tilde{\Delta}_i,i\in[n])$. Without SST, it is possible $\tilde{\Delta}_i < \Delta_i$, but if SST holds, then our algorithm is optimal up to a constant factor given $\delta \preceq 1/poly(n)$, or $\max_{i,j\in[n]}\tilde{\Delta}_i/\tilde{\Delta}_j \preceq O(n^{1/2-p})$ for some constant $p>0$. According to \cite{MaxingAndRanking2017,RankingLimits2018,Falahatgar2017}, ranking without the SST condition can be much harder than that with SST , and it remains an open problem whether our upper bound is tight or not when the SST condition does not hold.

\section{Numerical results}\label{Sec:NR}
In this section, we provide numerical results to demonstrate the efficacy of our proposed IIR algorithm. The code can be found in our GitHub page\footnote{https://github.com/WenboRen/ranking-from-noisy-comparisons}. 

We compare IIR with: (i) Active-Ranking (AR) \cite{ActiveRanking2019}, which focuses on the Borda-Score model and is not directly comparable to our algorithm. We use it as an example to show that although Borda-Score ranking may be the same as exact ranking, for finding the exact ranking, the performance of Borda-Score algorithms is not always as good as that for finding the Borda-Ranking \footnote{For instance, when $p_{r_i,r_j} = 1/2 + \Delta$ for all $i < j$, the Borda-Score of item $r_i$ is $\frac{1}{n-1}\sum_{j\neq i}p_{r_i,r_j}=1/2 + \frac{n+1-2i}{n-1}\Delta$, and $\Delta_{r_i} = \Theta(1/n)$. Thus, by \cite{ActiveRanking2019}, the sample complexity of AR is at least $O(n^3\log{n})$.}; (ii) PLPAC-AMPR \cite{OnlineRankingElicitation2015}, an algorithm for PAC ranking under the MNL model. By setting the parameter $\epsilon=0$, it can find the exact ranking with $O((n\log{n})\max_{i\in[n]}\Delta_i^{-2}\log(n\Delta_i^{-1}\delta^{-1}))$ comparisons, higher than our algorithm by at least a log factor; (iii) UCB + Binary Search of \cite{NoisyComputing1994}. In the Binary Search algorithm of \cite{NoisyComputing1994}, a subroutine that ranks two items with a constant confidence is required. In \cite{NoisyComputing1994}, it assumes the value of $\Delta_{\min} = \min_{i\in[n]}\Delta_i$ is priorly known, and the subroutine is simply comparing two items for $\Theta(\Delta_{\min}^{-2})$ times and returns the item that wins more. In this paper, the value of $\Delta_{\min}$ is not priorly known, and here, we use UCB algorithms such as LUCB \cite{KLLUCB2013} to play the role of the required subroutine. The UCB algorithms that we use include Hoeffding-LUCB \cite{Hoeffding1963,KLLUCB2013}, KL-LUCB \cite{KLConcentration1989,KLLUCB2013}, and lil'UCB \cite{LIL2014}. For Hoeffding-LUCB and KL-LUCB, we choose $\gamma = 2$. For lil'UCB, we choose $\epsilon = 0.01$, $\beta = 1$, and $\lambda = (\frac{2+\beta}{\beta})^2$.\footnote{We do not choose the combination ($\epsilon = 0$, $\beta = 1$, and $\lambda = 1 + 10/n$) that has a better practical performance because this combination does not have theoretical guarantee, making the comparison in some sense unfair.} Readers can find the source codes in supplementary material.

\textbf{Experiment Setup.} The experiments are conducted on three different types of instances. To simplify notation, we use $r_1\succ\!r_2\succ\!\cdots\succ\!r_n$ to denote the true ranking, and let $\Delta = 0.1$. (i) Type-Homo: For any $r_i\succ\!r_j$, $p_{r_i,r_j} = 1/2 + \Delta$. (ii) Type-MNL: The preference score of $r_i$ (i.e., $\theta_{r_i}$) is generated by taking an independent instance of Uniform$([0.9*1.5^{n-i}, 1.1*1.5^{n-i}])$. By this, for any $i$, $\Delta_i$ is around $0.1$. (iii) Type-Random: For any $r_i \succ\!r_j$, $p_{r_i,r_j}$ is generated by taking an independent instance of Uniform$([0.5 + 0.8\Delta, 0.5 + 1.5\Delta])$. By this, for any $i$, $\Delta_i$ is around $0.1$. 

The numerical results for these three types are presented in Figure~\ref{fig:NR}~(a)-(c), respectively. For all simulations, we input $\delta = 0.01$. Every point of every figure is averaged over 100 independent trials. In every figure, for the same $n$-value, the algorithms are tested on an identical input instance.

\begin{figure}[h]\centering
	\begin{subfigure}[b]{0.32\textwidth}
		\includegraphics[scale=0.5]{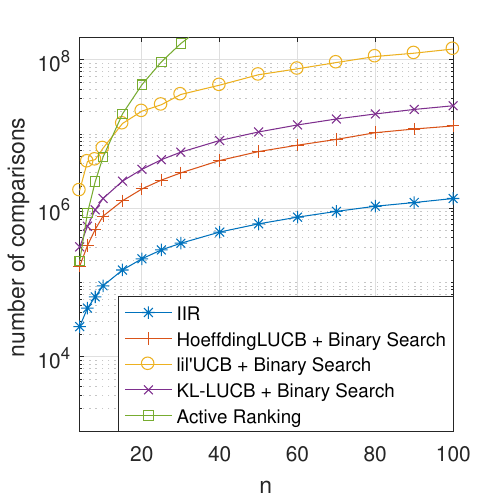}
		\caption{Type-Homo.}
	\end{subfigure}\ \ 
	\begin{subfigure}[b]{0.32\textwidth}
		\includegraphics[scale=0.5]{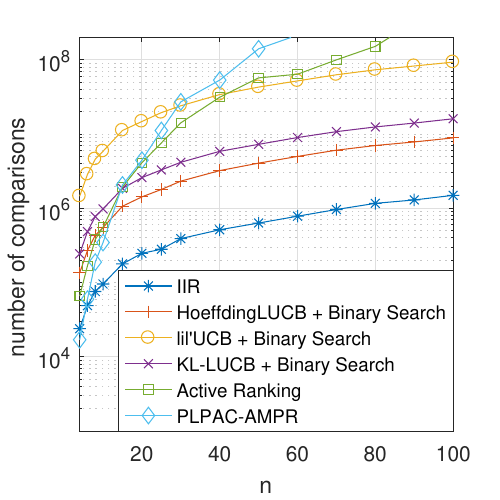}
		\caption{Type-MNL.}
	\end{subfigure}\ \ 
	\begin{subfigure}[b]{0.32\textwidth}
		\includegraphics[scale=0.5]{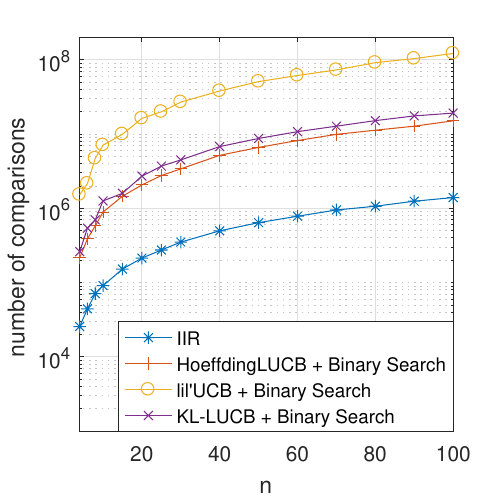}
		\caption{Type-Random.}
	\end{subfigure}
	\caption{Comparisons between IIR and existing methods.}\label{fig:NR}
\end{figure}

From Figure~\ref{fig:NR}, we can see that our algorithm significantly outperforms the existing algorithms. 
We can also see that the sample complexity of IIR scales with $n\log{n}$, which is consistent with our theoretical results. 
There are some insights about the practical performance of IIR. First, in Lines~3 and 4 of ATC and Lines~9 and 10 of ATI, we use LUCB-like \cite{KLLUCB2013} designs to allow the algorithms return before completing all required iterations, which does not improve the theoretical upper bound but can improve the practical performance.
Second, in the theoretical analysis, we only show that ATI correctly inserts an item $i$ with high probability when inputting $\epsilon\leq \Delta_i$, but the algorithm may return before $\epsilon$ being that small, making the practical performance better than what the theoretical upper bound suggests.  

\section{Conclusion}
In this paper, we investigated the theoretical limits of exact ranking with minimal assumptions. 
We do not assume any prior knowledge of the comparison probabilities and gaps, and derived the lower bounds and upper bound for instances with unequal noise levels. We also derived the model-specific pairwise and listwise lower bound for the MNL model, which further shows that in the worst case, listwise ranking is no more efficient than pairwise ranking in terms of sample complexity. 
The iterative-insertion-ranking (IIR) algorithm proposed in this paper indicates that our lower bounds are optimal under strong stochastic transitivity (SST) and some mild conditions. Numerical results suggest our ranking algorithm significantly outperforms existing works in the literature.

\bibliography{optimalranking}

\begin{thebibliography}{}

\bibitem[Agarwal et~al., 2017]{LimitedRounds2017}
Agarwal, A., Agarwal, S., Assadi, S., and Khanna, S. (2017).
\newblock Learning with limited rounds of adaptivity: {C}oin tossing,
  multi-armed bandits, and ranking from pairwise comparisons.
\newblock In {\em Conference on Learning Theory}, pages 39--75.

\bibitem[Arratia and Gordon, 1989]{KLConcentration1989}
Arratia, R. and Gordon, L. (1989).
\newblock Tutorial on large deviations for the binomial distribution.
\newblock {\em Bulletin of Mathematical Biology}, 51(1):125--131.

\bibitem[Baltrunas et~al., 2010]{RecommendationSystem2010}
Baltrunas, L., Makcinskas, T., and Ricci, F. (2010).
\newblock Group recommendations with rank aggregation and collaborative
  filtering.
\newblock In {\em ACM Conference on Recommender Systems}, pages 119--126. ACM.

\bibitem[Chen and Li, 2015]{ChenBestArmIdentification2015}
Chen, L. and Li, J. (2015).
\newblock On the optimal sample complexity for best arm identification.
\newblock {\em arXiv preprint arXiv:1511.03774}.

\bibitem[Chen et~al., 2013]{CrowdSourcing2013}
Chen, X., Bennett, P.~N., Collins-Thompson, K., and Horvitz, E. (2013).
\newblock Pairwise ranking aggregation in a crowdsourced setting.
\newblock In {\em ACM International Conference on Web Search and Data Mining},
  pages 193--202. ACM.

\bibitem[Chen et~al., 2018]{MNLlistwise2018}
Chen, X., Li, Y., and Mao, J. (2018).
\newblock A nearly instance optimal algorithm for top-k ranking under the
  multinomial logit model.
\newblock In {\em Proceedings of the Twenty-Ninth Annual ACM-SIAM Symposium on
  Discrete Algorithms}, pages 2504--2522. SIAM.

\bibitem[Chen et~al., 2019]{BothOptimal2017}
Chen, Y., Fan, J., Ma, C., and Wang, K. (2019).
\newblock Spectral method and regularized {MLE} are both optimal for top-k
  ranking.
\newblock {\em The Annals of Statistics}, 47(4):2204.

\bibitem[Chen and Suh, 2015]{SpectralMLE2015}
Chen, Y. and Suh, C. (2015).
\newblock Spectral {MLE}: {T}op-k rank aggregation from pairwise comparisons.
\newblock In {\em International Conference on Machine Learning}, pages
  371--380.

\bibitem[Conitzer and Sandholm, 2005]{SocialChoice2005Conitzer}
Conitzer, V. and Sandholm, T. (2005).
\newblock Communication complexity of common voting rules.
\newblock In {\em ACM Conference on Electronic Commerce}, pages 78--87. ACM.

\bibitem[Cover and Thomas, 1991]{InformationTheory1991}
Cover, T. and Thomas, J. (1991).
\newblock {\em Elements of Information Theory}.
\newblock John Wiley \& Sons.

\bibitem[Dwork et~al., 2001]{WebSearch2001}
Dwork, C., Kumar, R., Naor, M., and Sivakumar, D. (2001).
\newblock Rank aggregation methods for the web.
\newblock In {\em International Conference on World Wide Web}. ACM.

\bibitem[Falahatgar et~al., 2017a]{MaxingAndRanking2017}
Falahatgar, M., Hao, Y., Orlitsky, A., Pichapati, V., and Ravindrakumar, V.
  (2017a).
\newblock Maxing and ranking with few assumptions.
\newblock In {\em Advances in Neural Information Processing Systems}, pages
  7060--7070.

\bibitem[Falahatgar et~al., 2018]{RankingLimits2018}
Falahatgar, M., Jain, A., Orlitsky, A., Pichapati, V., and Ravindrakumar, V.
  (2018).
\newblock The limits of maxing, ranking, and preference learning.
\newblock In {\em International Conference on Machine Learning}, pages
  1427--1436. PMLR.

\bibitem[Falahatgar et~al., 2017b]{Falahatgar2017}
Falahatgar, M., Orlitsky, A., Pichapati, V., and Suresh, A.~T. (2017b).
\newblock Maximum selection and ranking under noisy comparisons.
\newblock In {\em International Conference on Machine Learning}, pages
  1088--1096.

\bibitem[Fano and Wintringham, 1961]{Fano1961}
Fano, R.~M. and Wintringham, W. (1961).
\newblock Transmission of information.
\newblock {\em Physics Today}, 14(12):56.

\bibitem[Farrell, 1964]{RatioTest1964}
Farrell, R.~H. (1964).
\newblock Asymptotic behavior of expected sample size in certain one sided
  tests.
\newblock {\em The Annals of Mathematical Statistics}, pages 36--72.

\bibitem[Feige et~al., 1994]{NoisyComputing1994}
Feige, U., Raghavan, P., Peleg, D., and Upfal, E. (1994).
\newblock Computing with noisy information.
\newblock {\em SIAM Journal on Computing}, 23(5):1001--1018.

\bibitem[Heckel et~al., 2019]{ActiveRanking2019}
Heckel, R., Shah, N.~B., Ramchandran, K., Wainwright, M.~J., et~al. (2019).
\newblock Active ranking from pairwise comparisons and when parametric
  assumptions do not help.
\newblock {\em The Annals of Statistics}, 47(6):3099--3126.

\bibitem[Heckel et~al., 2018]{ApproximateRanking2018}
Heckel, R., Simchowitz, M., Ramchandran, K., and Wainwright, M.~J. (2018).
\newblock Approximate ranking from pairwise comparisons.
\newblock In {\em International Conference on Artificial Intelligence and
  Statistics}, pages 1057--1066.

\bibitem[Hoeffding, 1994]{Hoeffding1963}
Hoeffding, W. (1994).
\newblock Probability inequalities for sums of bounded random variables.
\newblock In {\em The Collected Works of Wassily Hoeffding}, pages 409--426.
  Springer.

\bibitem[Jamieson et~al., 2014]{LIL2014}
Jamieson, K., Malloy, M., Nowak, R., and Bubeck, S. (2014).
\newblock lil{'UCB}: {A}n optimal exploration algorithm for multi-armed
  bandits.
\newblock In {\em Conference on Learning Theory}, pages 423--439.

\bibitem[Jang et~al., 2017]{ListwisePL2017}
Jang, M., Kim, S., Suh, C., and Oh, S. (2017).
\newblock Optimal sample complexity of {M}-wise data for top-{K} ranking.
\newblock In {\em Advances in Neural Information Processing Systems}, pages
  1686--1696.

\bibitem[Kalyanakrishnan and Stone, 2010]{Halving2010}
Kalyanakrishnan, S. and Stone, P. (2010).
\newblock Efficient selection of multiple bandit arms: {T}heory and practice.
\newblock In {\em International Conference on Machine Learning}, pages
  511--518.

\bibitem[Kalyanakrishnan et~al., 2012]{LowerBound2012}
Kalyanakrishnan, S., Tewari, A., Auer, P., and Stone, P. (2012).
\newblock {PAC} subset selection in stochastic multi-armed bandits.
\newblock In {\em International Conference on Machine Learning}, pages
  227--234.

\bibitem[Katariya et~al., 2018]{CoarseRanking2018}
Katariya, S., Jain, L., Sengupta, N., Evans, J., and Nowak, R. (2018).
\newblock Adaptive sampling for coarse ranking.
\newblock In {\em International Conference on Artificial Intelligence and
  Statistics}, pages 1839--1848.

\bibitem[Kaufmann and Kalyanakrishnan, 2013]{KLLUCB2013}
Kaufmann, E. and Kalyanakrishnan, S. (2013).
\newblock Information complexity in bandit subset selection.
\newblock In {\em Conference on Learning Theory}, pages 228--251. PMLR.

\bibitem[Luce, 2012]{Luce2012}
Luce, R.~D. (2012).
\newblock {\em Individual choice behavior: {A} theoretical analysis}.
\newblock Courier Corporation.

\bibitem[Mannor and Tsitsiklis, 2004]{FKLowerBound2004}
Mannor, S. and Tsitsiklis, J.~N. (2004).
\newblock The sample complexity of exploration in the multi-armed bandit
  problem.
\newblock {\em Journal of Machine Learning Research}, 5(Jun):623--648.

\bibitem[Mohajer and Suh, 2016]{Mohajer2016active}
Mohajer, S. and Suh, C. (2016).
\newblock Active top-k ranking from noisy comparisons.
\newblock In {\em Annual Allerton Conference on Communication, Control, and
  Computing (Allerton)}, pages 875--882. IEEE.

\bibitem[Negahban et~al., 2017]{RankCentrarity2016}
Negahban, S., Oh, S., and Shah, D. (2017).
\newblock Rank centrality: {R}anking from pairwise comparisons.
\newblock {\em Operations Research}, pages 266--287.

\bibitem[Pfeiffer et~al., 2012]{AdaptivePooling2012}
Pfeiffer, T., Xi, A., Gao, A., Mao, Y., Chen, and Rand, D.~G. (2012).
\newblock Adaptive polling for information aggregation.
\newblock In {\em {AAAI} Conference on Artificial Intelligence}.

\bibitem[Ren et~al., 2018]{RankingBounds2018}
Ren, W., Liu, J., and Shroff, N.~B. (2018).
\newblock {PAC} ranking from pairwise and listwise queries: {L}ower bounds and
  upper bounds.
\newblock {\em arXiv preprint arXiv:1806.02970}.

\bibitem[Ren et~al., 2019]{QuantileBandit2019}
Ren, W., Liu, J., and Shroff, N.~B. (2019).
\newblock Exploring $ k $ out of top $\rho$ fraction of arms in stochastic
  bandits.
\newblock In {\em International Conference on Artificial Intelligence and
  Statistics}, pages 2820--2828. PMLR.

\bibitem[Saha and Gopalan, 2019a]{Subsetwise2019}
Saha, A. and Gopalan, A. (2019a).
\newblock Active ranking with subset-wise preferences.
\newblock In {\em International Conference on Artificial Intelligence and
  Statistics}, pages 3312--3321.

\bibitem[Saha and Gopalan, 2019b]{ListwisePLMaxing2019}
Saha, A. and Gopalan, A. (2019b).
\newblock From {PAC} to instance-optimal sample complexity in the
  {P}lackett-{L}uce model.
\newblock {\em arXiv preprint arXiv:1903.00558}.

\bibitem[Shah et~al., 2016]{TransitivityModel2016}
Shah, N., Balakrishnan, S., Guntuboyina, A., and Wainwright, M. (2016).
\newblock Stochastically transitive models for pairwise comparisons:
  Statistical and computational issues.
\newblock In {\em International Conference on Machine Learning}, pages 11--20.
  PMLR.

\bibitem[Shah and Wainwright, 2017]{Simple2017}
Shah, N.~B. and Wainwright, M.~J. (2017).
\newblock Simple, robust and optimal ranking from pairwise comparisons.
\newblock {\em Journal of Machine Learning Research}, 18(1):7246--7283.

\bibitem[Szörényi et~al., 2015]{OnlineRankingElicitation2015}
Szörényi, B., Busa-Fekete, R., Paul, A., and Hüllermeier, E. (2015).
\newblock Online rank elicitation for {P}lackett-{L}uce: {A} dueling bandits
  approach.
\newblock In {\em Advances in Neural Processing Systems}, pages 604--612.

\end{thebibliography}
\clearpage

\begin{appendices}
	
	{\LARGE Supplementary material}
	
	\section{Further discussions}\label{Sec:FRDiscussions}
	
	\subsection{Non-$\delta$-correct algorithms}
	In Section~\ref{Sec:Intro}, we define the notion of $\delta$-correct algorithms, which return correct results with probability at least $1-\delta$ for any input instances satisfying assumptions A1 to A3 (defined in Section~\ref{Sec:Intro}). It is reasonable to consider $\delta$-correct algorithms since we may not want an algorithm that performs pretty well on some instances but badly on others. However, to give better insights about $\delta$-correct algorithms and the lower bounds in Theorem~\ref{Theorem:LB}, we give an algorithm that is not $\delta$-correct and has sample complexity lower than Theorem~\ref{Theorem:LB} for a specific class of instances.
	
	\begin{example}[A non-$\delta$-correct algorithm]\label{Example:non}
		$\mathcal{A}$ is an algorithm for ranking $3$ items. It views each pair of items as a coin, and calls KL-LUCB \cite{KLLUCB2013} to find the pair $(i,j)$ with the largest $p_{i,j}$-value. Then, it claims that $i$ is the most preferred item and $j$ is the worst. Obviously, $\mathcal{A}$ is not $\delta$-correct for ranking $3$ items. However, for an instance with $p_{r_1,r_2}=1/2+\Delta$, $p_{r_1,r_3}=1-\Delta$, and $p_{r_2,r_3}=1-2\Delta$, where $r_1\succ r_2\succ r_3$ is the unknown true ranking and $\Delta\in(0,1/6)$ is unknown, with probability at least $1-\delta$, algorithm $\mathcal{A}$ finds its true ranking by using $O({\Delta^{-1}}\log(\Delta^{-1}\delta^{-1}))$ comparisons.
	\end{example}
	
	To see this upper bound, we first define some notations. For $p,q\in[0,1]$, the KL-Divergence \cite{InformationTheory1991} between them is defined as $d(q,p) := D_{KL}(q||p)=q\log\frac{q}{p}+(1-q)\log\frac{1-q}{1-p}$. The Chernoff-Information \cite{KLLUCB2013} between them is defined as $d^*(q,p):=d(z^*,p)=d(z^*,q)$, where $z^*$ is the unique $z$ such that $d(z,p)=d(z,q)$. According to \cite[Theorem~3]{KLLUCB2013}, the algorithm KL-LUCB distinguishes two coins (Bernoulli arms) with mean rewards $\lambda$ and $\mu$ by taking $O(\frac{1}{d^*(\lambda,\mu)}\log\frac{1}{\delta d^*(\lambda,\mu)})$ samples. In this instance, we observe that for a constant $c>1$, $d(c\Delta,\Delta)=\Theta(\Delta)$. Thus, we have $d^*(1-2\Delta,1-\Delta) = d^*(2\Delta,\Delta)=\Theta(\Delta)$. Hence, KL-LUCB distinguishes $p_{r_1,r_3}$ and $p_{r_2,r_3}$ by $O(\Delta^{-1}\log(\delta^{-1}\Delta^{-1}))$ comparisons. Since the gap between $p_{r_1,r_2}$ and $p_{r_1,r_3}$ is even larger, they can also be distinguished by the above number of comparisons. This shows the upper bound, which suggests that the $\Delta_{i}^{-2}$ term is not necessary for non-$\delta$-correct algorithms.
	
	We note that $\mathcal{A}$ does not need any information of this instance a priori to run. Although it is not $\delta$-correct, it can solve this class of instances with sample complexity lower than Theorem~\ref{Theorem:LB}. However, in general, this algorithm may be of no sense as it only works for a restricted class of instances. This is the reason why we want to bound the sample complexity of $\delta$-correct algorithms but not that of arbitrary ones, as there may always exist non-$\delta$-correct algorithms that have extremely good performance on some restricted class of instances.
	
	
	\subsection{An instance where Eq. (\ref{Eq:LB}) does not hold as a lower bound}
	When $\delta$ is a positive constant and $\max_{i,j\in[n]}\tilde{\Delta}_i/\tilde{\Delta}_j \succ \sqrt{n}$, the lower bound given in Eq.~(\ref{Eq:LB}) may not hold. In this subsection, we give an example such that Eq.~(\ref{Eq:LB} does not hold as a lower bound.
	
	\begin{example}[An example that Eq.~(\ref{Eq:LB}) does not hold as a lower bound]
		Assume that $r_1\succ\!r_2\succ\!\cdots\succ\!r_n$ is the unknown true ranking. Suppose $\delta = 1/4$, $\Delta_{r_1,r_2} = n^{-10}$ and $\Delta_{r_i,r_j} = 0.01$ for all $\{r_i,r_j\} \neq \{r_1,r_2\}$. For this instance, there is a $(1/4)$-correct algorithm that finds its true ranking with confidence $3/4$ by $O(n^{20}\log\log{n}+n^2\log{n})$ comparisons, which is lower than Eq.~(\ref{Eq:LB}): $\tilde{\Omega}(n^{20}\log{n} + n\log{n})$. This implies that Eq.~(\ref{Eq:LB}) does not hold as a lower bound in this case.
	\end{example}
	
	To see the upper bound, we can view each pair as a coin (aka Bernoulli arms), and then use lil'UCB \cite{LIL2014} to find the pair with the least gap (i.e., $\Delta_{i,j}$) with confidence $11/12$. According to \cite{LIL2014}, this step takes $O(n)$ comparisons. Then, we rank the pair with the smallest gap with $11/12$ confidence. This step takes $O(\Delta_{r_1,r_2}^{-2}\log\log\Delta_{r_1,r_2}^{-1}) = O(n^{20}\log\log{n})$ comparisons. Finally, we rank all other pairs with $1-\frac{1}{12n^2}$ confidence for each, and this step takes $O(n^2\log{n})$ comparisons. After ranking all pairs of items, the true ranking is found, and thus, the total sample complexity is $O(n^{20}\log\log{n}+n^2\log{n})$.
	
	For this instance, the lower bound in Eq. (\ref{Eq:LB}) is $\tilde{\Omega}(n^{20}\log{n} + n\log{n})$, higher than the upper bound. Thus, when the given condition does not hold, the lower bound in Eq. (\ref{Eq:LB}) may not hold. However, there is at most a log gap, and the lower bound in Eq. (\ref{Eq:GLB}) does not need this condition.

\section{Proofs}\label{Sec:FRProofs}

\subsection{Proof of Theorem~\ref{Theorem:LB}}

\RestateLB*

\begin{proof}
	\textbf{Step 1} is to prove the lower bound for ranking two items, which is stated in Lemma~\ref{Lemma:LBTwo}. In the proof of Lemma~\ref{Lemma:LBTwo}, we will make use of the results in \cite{RatioTest1964,LIL2014,FKLowerBound2004}. The proof can be found in Section~\ref{Sec:LB2}
	
	\begin{restatable}[Lower bound for ranking two items]{lemma}{RestateLBTwo}\label{Lemma:LBTwo}
		Let $\delta\in(0,1/4)$ and $\delta$-correct algorithm $\mathcal{A}_2$ be given. Let $T_{\mathcal{A}_2}(\Delta_{i,j})$ be the number of comparisons conducted by $\mathcal{A}_2$ under the $\Delta_{i,j}$-values. To rank $i$ and $j$ with error probability no more than $\delta$, there is a universal constant $c_{lb2}>0$ such that
		\begin{align}\label{Eq:LogLogTwo}
		\limsup_{\Delta_{i,j}\rightarrow 0}\frac{\mathbb{E}[T_{\mathcal{A}_2}(\Delta_{i,j})]}{\Delta_{i,j}^{-2}(\log\log\Delta_{i,j}^{-2} + \log{\delta^{-1}})} \geq c_{lb2}.
		\end{align}
	\end{restatable}
	
	The authors of \cite{ChenBestArmIdentification2015} stated a stronger sample complexity lower bound for identifying the sign of the mean of a Gaussian arm than \cite{LIL2014} in their Theorem~D.1. By using Theorem~D.1 of \cite{ChenBestArmIdentification2015} and the same reduction from Gaussian arms to ranking from two items as that in the proof of Lemma~\ref{Lemma:LBTwo}, we get a stronger lower bound by similar steps, which is stated in Corollary~\ref{Corollary:LBTwo2}.
	\begin{corollary}[A stronger lower bound for ranking two items]\label{Corollary:LBTwo2}
	    Let $\delta\in(0,1/4)$ and $\delta$-correct algorithm $\mathcal{A}'_2$ be given. Let $T_{\mathcal{A}'_2}(\Delta_{i,j})$ be the number of comparisons conducted by $\mathcal{A}'_2$ under the $\Delta_{i,j}$-values. We have 
	    \begin{align}
	        \mathbb{E}[T_{\mathcal{A}'_2}(\Delta_{i,j})] = \tilde{\Omega}(\Delta_{i,j}^{-2}\log\log\Delta_{i,j}^{-1}) + \Omega(\Delta_{i,j}^{-2}\log\delta^{-1}).
	    \end{align}
	\end{corollary}
	
	Here we note that for two functions $f(\mathbf{x})$ and $g(\mathbf{y})$ where $\mathbf{x} \in \mathbb{R}^k$ and $\mathbf{y} \in \mathbb{R}^l$, we have 
	\begin{align}
	    \tilde{\Omega}(f(\mathbf{x})) + \tilde{\Omega}(g(\mathbf{y})) = \tilde{\Omega}(f(\mathbf{x} + g(\mathbf{y}))). \label{Eq:TildeAdditivity}
	\end{align}
	The proof is not complex. We let $\lambda(\mathbf{x}) = \tilde{\Omega}(f(\mathbf{x})) $ and $\mu(\mathbf{y}) = \tilde{\Omega}(g(\mathbf{y}))$. We use $X_N$ to denote the size of 
	\begin{align}
	    \{(i_1, i_2,...,i_k) \in [N]^k : \exists \mathbf{x} \in E_{i_1}\times E_{i_2} \times \cdots \times E_{i_k} \mbox{ such that } \lambda(\mathbf{x}) < c_0 f(\mathbf{x}) \}, \nonumber
	\end{align}
	and $Y_N$ be the similar thing for $\mu(\mathbf{y})$ and $g(\mathbf{y})$. 
	
	According to the definition of $\tilde{\Omega}(\cdot)$, for any $\gamma > 0$ we have
	\begin{align}
	    \lim_{N\rightarrow\infty} \frac{X_N}{N^{k-1+\gamma}} = 0, \mbox{ and } \lim_{N\rightarrow\infty} \frac{Y_N}{N^{l-1+\gamma}} = 0. \nonumber
	\end{align}
	Also, the size of 
	\begin{align}
	    \{& (i_1, i_2,...,i_k, j_1, j_2,..., j_l) \in [N]^{k+l} : \nonumber \\
	    & \exists \mathbf{x} \in E_{i_1}\times E_{i_2} \times \cdots \times E_{i_k} \mbox{ such that } \lambda(\mathbf{x}) < c_0 f(\mathbf{x}) \nonumber \\
	    & \mbox{ or } \exists \mathbf{y} \in E_{j_1}\times E_{j_2} \times \cdots \times E_{j_l} \mbox{ such that } \mu(\mathbf{y}) < c_0 g(\mathbf{y})\} \nonumber
	\end{align}
	is upper bounded by $X_N N^l + Y_N N^k$, which, for any $\gamma > 0$, has
	\begin{align}
	    \lim_{N \rightarrow \infty} \frac{X_N N^l + Y_N N^k}{N^{k + l - 1 + \gamma}} = 0. \nonumber
	\end{align}
	Therefore, Eq~(\ref{Eq:TildeAdditivity}) holds.
	
	\textbf{Step 2 }is to define problems $\mathcal{P}_1$ and $\mathcal{P}_2$. Let $(r_1,r_2,...,r_n)$ be a given permutations of $[n]$ and assume that $q_1\succ\!q_2\succ\!\cdots\succ\!q_n$ is the unknown true ranking. Assume that $n$ is odd (when $n$ is even, we can prove the same results similarly), and say $n=2m+1$. A pair $(r_i, r_j)$ is said to be \textit{significant} if there exists an $k$ in $[m]$ such that $\{r_i,r_j\} = \{r_{2k-1},r_{2k}\}$, and \textit{insignificant} otherwise.
	
	Define a set $\Pi:=\{0,1\}^m$. For any $\vec{\pi}=(\pi_1,\pi_2,...,\pi_m)\in\Pi$, define a corresponding hypothesis $\mathcal{H}_{\vec{\pi}}$ that claims: (i) the true ranking of $[n]$ is $s_1\succ\!s_2\succ\!\cdots\succ\!s_n$; (ii) $s_n=r_n$; (iii) for any $k\in[m]$, $(s_{2k-1},s_{2k})=(r_{2k-1},r_{2k})$ if $\pi_k=1$, and $(s_{2k-1},s_{2k})=(r_{2k},r_{2k-1})$ otherwise; (iv) for any insignificant pair $(r_i,r_j)$, the probability that $r_i$ wins a comparison over the pair $(r_i,r_j)$ is $p^{\vec{\pi}}_{r_i,r_j} = p_{r_i,r_j}$; (v) For any $k\in[m]$ and the corresponding significant pair $(r_{2k-1},r_{2k})$, the probability that $r_{2k-1}$ wins a comparison over the pair $(r_{2k-1},r_{2k})$ is $p^{\vec{\pi}}_{r_i,r_j} = 1/2 + \Delta_{r_{2k-1},r_{2k}}$ if $\pi_k = 1$, and is $(1/2 - \Delta_{r_{2k-1},r_{2k}})$ otherwise. In other words, $\mathcal{H}_{\vec{\pi}}$ claims a true ranking that is almost the same as $r_1\succ\!r_2\succ\!\cdots\succ\!r_n$ but the positions of $(r_{2k-1},r_{2k})$ are exchanged for all $k\in[m]$ such that ${\pi}_k=0$. E.g., for $n = 3$ and $\vec{\pi} = (0)$, $\mathcal{H}_{\vec{\pi}}$ claims that the true ranking is $r_2\succ\!r_1\succ\!r_3$, $p^{\vec{\pi}}_{r_1,r_2} = 1/2 - \Delta_{r_1,r_2}$, $p^{\vec{\pi}}_{r_1,r_3} = p_{r_1,r_3}$, and $p^{\vec{\pi}}_{r_2,r_3} = p_{r_2,r_3}$. 
	
	We further assume that there is a $\vec{\pi}_0\in\Pi$ such that $\mathcal{H}_{\vec{\pi}_0}$ is true, and each $\vec{\pi}\in\Pi$ has the same prior probability to be $\vec{\pi}_0$.
	
	\textbf{Problem $\mathbf{\mathcal{P}_1}$.}
	Knowing the fact that there exists a $\mathbf{\pi}^0\in \Pi$ such that $\mathcal{H}_{\vec{\pi}^0}$ is true, we want to find $\mathbf{\pi}^0$ with confidence $1-\delta$, and use as few comparisons as possible.
	
	Next, we start defining problem $\mathcal{P}_2$. An instance of $\mathcal{P}_2$ involves ${n \choose 2}$ coins, and each is indexed by an element of $\{(i,j):i,j\in[n]\land i < j\}$. We use $C_{i,j}$ to denote the coin indexed by $(i,j)$. For each coin $C_{i,j}$, each toss of it gives a head with probability $\mu_{i,j}$, and gives a tail with probability $1-\mu_{i,j}$. We name $\mu_{i,j}$ as the \textit{head probability} of coin $C_{i,j}$. We assume that the outcomes of tosses are independent across coins and time. Similar to the items, coin $C_{i,j}$ is said to be \textit{significant} if there is a $k$ such that $(i,j)=(2k-1,2k)$, and is \textit{insignificant} otherwise. We assume that for all insignificant coins $C_{i,j}$, $\mu_{i,j} = p_{r_i,r_j}$, and for all significant coins $C_{2k-1,2k}$, $\mu_{2k-1,2k} =  1/2 + \Delta_{r_{2k-1},r_{2k}}$ or $1/2 - \Delta_{r_{2k-1},r_{2k}}$, either has a prior probability $1/2$ to be true. 
	
	\textbf{Problem $\mathcal{P}_2$.}
	With probability $\geq 1-\delta$, we want to find whether $\mu_{2k-1,2k}>1/2$ for all $k\in[m]$.
	
	\textbf{Step 3} is to show the following lemma, which states that $\mathcal{P}_2$ can be reduced to $\mathcal{P}_1$, and $\mathcal{P}_1$ can be reduced to exact ranking. Its proof can be found in Section~\ref{Sec:PfReductions}.
	
	\begin{restatable}[Reductions]{lemma}{RestateReductions}\label{Lemma:P12Reductions}
		With the above definitions, (i) if the true ranking of $[n]$ is found, with no more comparisons, one can get the solution of $\mathcal{P}_1$, and (ii) if an algorithm solves $\mathcal{P}_1$ with $N$ expected number of comparison, there is another algorithm that solves $\mathcal{P}_2$ with $N$ expected number of tosses.
	\end{restatable}

	\textbf{Step 4} is to prove the following lemma regarding the lower bound of problem $\mathcal{P}_2$. Its proof can be found in Section~\ref{Sec:PfLBP2}
	
	\begin{restatable}{lemma}{RestateLBPtwo}\label{Lemma:LBP2}
		For $\delta\in(0,1/12)$, the expected number of tosses needed for solving $\mathcal{P}_2$ is at least
		\begin{align}\label{Eq:LBP2_O1}
		&\tilde{\Omega}\Big(\sum_{k\in[m]}\Delta_{q_{2k-1},q_{2k}}^{-2}\cdot\log\log\Delta_{q_{2k-1},q_{2k}}^{-1}\Big) \nonumber \\
		& + \Omega\Big(\min\{\sum_{k\in[m]}{\Delta_{q_{2k-1},q_{2k}}^{-2}\cdot\log(\delta^{-1}_k)}:\sum_{k\in[m]}{\delta_k} \leq 2\delta\}\Big).
		\end{align}
	\end{restatable}
	
	\textbf{Step 5} is to prove the lower bound given in Eq. (\ref{Eq:GLB}). Lemmas~\ref{Lemma:P12Reductions} proves that we can reduce $\mathcal{P}_2$ to $\mathcal{P}_1$ and reduce $\mathcal{P}_1$ to exact ranking. Lemma~\ref{Lemma:LBP2} states a lower bound on $\mathcal{P}_2$. Thus, by Lemmas~\ref{Lemma:P12Reductions} and \ref{Lemma:LBP2}, we have that the sample complexity of exact ranking is lower bounded by (\ref{Eq:LBP2_O1}).
	
	We can construct a similar problem to $\mathcal{P}_2$, and by the similar steps as in the proof of Lemma~\ref{Lemma:LBP2}, we have that the sample complexity of exact ranking is also lower bounded by
	\begin{align}\label{Eq:LBP2_v2}
	& \tilde{\Omega}\Big(\sum_{k\in[m]}\Delta_{q_{2k},q_{2k+1}}^{-2}\cdot\log\log\Delta_{q_{2k},q_{2k+1}}^{-1}\Big) \nonumber \\
	& + \Omega\Big(\min\{\sum_{k\in[m]}{\Delta_{q_{2k},q_{2k+1}}^{-2}\log(1/\delta_k)}: \sum_{k\in[m]}{\delta_k} \leq 2\delta\}\Big).
	\end{align}
	
	We recall that $q_1\succ\!q_2\succ\!\cdots\succ\!q_n$ is the true ranking. Since for any $i\in[n]$, $\tilde{\Delta}_{q_i} =  \Delta_{q_{i},q_{i-1}} \land \Delta_{q_{i},q_{i+1}}$, we have
	\begin{align}
	\mathbb{E}N_\mathcal{A}(\mathcal{I}) \tilde{\succeq} & \sum_{k\in[m]}\Delta_{q_{2k-1},q_{2k}}^{-2}\log\log\Delta_{q_{2k-1},q_{2k}}^{-1} + \sum_{k\in[m]}\Delta_{q_{2k},q_{2k+1}}^{-2}\log\log\Delta_{q_{2k},q_{2k+1}}^{-1} \nonumber \\
	\geq & \sum_{k\in[m]}\max\{\Delta_{q_{2k-1},q_{2k}}^{-2}\log\log\Delta_{q_{2k-1},q_{2k}}^{-1}, \Delta_{q_{2k},q_{2k+1}}^{-2}\log\log\Delta_{q_{2k},q_{2k+1}}^{-1}\}\nonumber \\ 
	= & \sum_{k\in[m]}{\tilde{\Delta}_{q_{2k}}^{-2}\log\log\tilde{\Delta}_{q_{2k}}^{-1}} \nonumber \\
	\stackrel{(a)}{\geq} & \frac{1}{3}\sum_{i=1}^{n}{\tilde{\Delta}_i^{-2}\log\log\tilde\Delta_i^{-1}},\label{Eq:LBF1}
	\end{align}
	where (a) holds because for any $k\in[m]$, $\tilde{\Delta}_{q_{2k+1}} =  \Delta_{q_{2k},q_{2k+1}} \land \Delta_{q_{2k+1},q_{2k+2}} \geq \tilde{\Delta}_{q_{2k}} \land \tilde{\Delta}_{q_{2k+2}}$. 
	
	We also have
	\begin{align}
	& \min\Big\{\sum_{k\in[m]}{\Delta_{q_{2k-1},q_{2k}}^{-2}\log\delta_k^{-1}}:\sum_{k\in[m]}{\delta_k} \leq 2\delta\Big\} \nonumber \\
	& \quad\quad\quad\quad + \min\Big\{\sum_{k\in[m]}{\Delta_{q_{2k},q_{2k+1}}^{-2}\log\delta_k^{-1}}: \sum_{k\in[m]}{\delta_k} \leq 2\delta\Big\}\nonumber \\
	& = \min\Big\{\sum_{k\in[m]}[\Delta_{q_{2k-1},q_{2k}}^{-2}\log(1/\delta_k) + \Delta_{q_{2k},q_{2k+1}}^{-2}\log(1/\delta'_k)]: \nonumber \\
	&\quad\quad\quad\quad \sum_{k\in[m]}\delta_k\leq 2\delta, \sum_{k\in[m]}\delta'_k \leq 2\delta\Big\}\nonumber\\
	& \geq \min\Big\{\sum_{k\in[m]}[\Delta_{q_{2k-1},q_{2k}}^{-2} + \Delta_{q_{2k},q_{2k+1}}^{-2}]\log\frac{1}{\delta \lor \delta'} : \sum_{k\in[m]}\delta_k \leq 2\delta, \sum_{k\in[m]}\delta'_k \leq 2\delta\Big\} \nonumber \\
	& \geq \min\Big\{\sum_{k\in[m]}\tilde{\Delta}_{q_{2k}}^{-2}\log\frac{1}{\delta_k \lor \delta'_k} : \sum_{k\in[m]}\delta_k \leq 2\delta, \sum_{k\in[m]}\delta'_k \leq 2\delta\Big\} \nonumber \\
	& \geq \min\Big\{\sum_{k\in[m]}\tilde{\Delta}_{q_{2k}}^{-2}\log\frac{1}{\delta_k \lor \delta'_k} : \sum_{k\in[m]}\delta_k\lor\delta'_k \leq 4\delta\Big\} \nonumber \\
	& \geq \min\Big\{\sum_{k\in[m]}\tilde{\Delta}_{q_{2k}}^{-2}\log(1/\delta_k) : \sum_{k\in[m]}\delta_k \leq 4\delta\Big\} \nonumber \\
	& \geq \min\Big\{\frac{1}{3}\sum_{i\in[n]}{\tilde{\Delta}_i^{-2}\log(1/x_i)}: \sum_{i\in[n]}x_i\leq 12\delta\Big\}.\label{Eq:LB13}
	\end{align}
	
	By (\ref{Eq:LB13}), first, we obtain that, for all $\delta\in(1,1/12)$, 
	\begin{align}\label{Eq:LBF2}
	\mathbb{E}N_\mathcal{A}(\mathcal{I}) \succeq \sum_{i\in[n]}\tilde{\Delta}_i^{-2}\log(1/\delta).
	\end{align}
	
	Also, since $\delta < 1/12$, we obtain the lower bound 
	\begin{align}\label{Eq:LBF3}
	\mathbb{E}N_\mathcal{A}(\mathcal{I}) \succeq & \min\{\sum_{i\in[n]}{\tilde{\Delta}_i^{-2}\log(1/x_i)}: \sum_{i\in[n]}x_i\leq 12\delta\} \nonumber \\
	\geq & \min\{\sum_{i\in[n]}{\tilde{\Delta}_i^{-2}\log(1/x_i)}: \sum_{i\in[n]}x_i\leq 1\}.
	\end{align}
	
	The lower bound in Eq. (\ref{Eq:GLB}) follows from summing up Equations~(\ref{Eq:LBF1}), (\ref{Eq:LBF2}), and (\ref{Eq:LBF3}). This prove the lower bound in Eq.~(\ref{Eq:LB}).
	
	\textbf{Step 6} is to deduce the lower bound in Eq. (\ref{Eq:LB}) from Eq.~(\ref{Eq:GLB}). 
	
	\textbf{Case 1.} We consider the cases where $\delta \preceq 1/poly(n)$. We observe that, when $\delta \preceq 1/poly(n)$, $\log(1/\delta) \succeq \log{n}$. Thus, in Eq.~(\ref{Eq:LBF3}), setting all $x_i = 1/n$, we have
	\begin{align}
	\min\{\sum_{i\in[n]}{\tilde{\Delta}_i^{-2}\log(1/x_i)}: \sum_{i\in[n]}x_i\leq 1\} \leq \sum_{i\in[n]}\tilde{\Delta}_i^{-2}\log{n} \preceq \sum_{i\in[n]}\tilde{\Delta}_i^{-2}\log(1/\delta).\nonumber
	\end{align}
	This means that the term $\min\{\cdots\}$ is dominated by the term $\sum_{i\in[n]}\tilde{\Delta}_i^{-2}\log(1/\delta)$. We also have $\sum_{i\in[n]}\tilde{\Delta}_i^{-2}\log(1/\delta) \simeq \sum_{i\in[n]}\tilde{\Delta}_i^{-2}\log(n/\delta)$ since $\log{\delta^{-1}} \succeq \log{n}$. Thus, 
	\begin{align}
	\sum_{i\in[n]}\tilde{\Delta}_i^{-2}\log(1/\delta) + \min\{\sum_{i\in[n]}{\tilde{\Delta}_i^{-2}\log(1/x_i)}: \sum_{i\in[n]}x_i\leq 1\} \simeq \sum_{i\in[n]}\tilde{\Delta}_i^{-2}\log(n/\delta),\nonumber
	\end{align}
	which implies that when $\delta = 1/poly(n)$, the lower bound in (\ref{Eq:LB}) holds.
	
	\textbf{Case 2.} We consider the case where $\max_{i,j\in[n]}\{\tilde{\Delta}_i/\tilde{\Delta}_j\} \leq c\cdot n^{1/2-p}$ for some constants $c,p>0$. When this condition holds, for any $x_1,x_2,...,x_n$ with $\sum_{i\in[n]}x_i\leq 1$, we have
	\begin{align}
	\sum_{i\in[n]}{\tilde{\Delta}_i^{-2}\log(1/x_i)} = & {\sum_{j\in[n]}\tilde{\Delta}_j^{-2}} \sum_{i\in[n]}\frac{\tilde{\Delta}_i^{-2}}{\sum_{j\in[n]}\tilde{\Delta}_j^{-2}}\cdot\log(1/x_i) \nonumber \\
	\stackrel{(a)}{\geq} & \sum_{j\in[n]}{\tilde{\Delta}_j^{-2}}\cdot \log\frac{1}{\sum_{i\in[n]}x_i\cdot \frac{\tilde{\Delta}_i^{-2}}{\sum_{j\in[n]}{\tilde{\Delta}_j^{-2}}}} \nonumber \\
	\geq & \sum_{j\in[n]}{\tilde{\Delta}_j^{-2}}\cdot \log\frac{1}{\sum_{i\in[n]}x_i\frac{1}{\sum_{j\in[n]}(c\cdot n^{-1/2+p})^2}} \nonumber \\
	\stackrel{(b)}{\geq} & \sum_{j\in[n]}{\tilde{\Delta}_j^{-2}} \log\Big[{\sum_{i\in[n]}(c\cdot n^{-1/2+p})^2}\Big] \nonumber \\
	\geq & \sum_{j\in[n]}{\tilde{\Delta}_j^{-2}}\log(c^2n^{2p}) \nonumber \\
	\succeq & \sum_{i\in[n]}{\tilde{\Delta}_i^{-2}}\log{n},\nonumber
	\end{align}
	where (a) is due to the convexity of the functions $(\log(1/x_i), i\in[n])$, and (b) is due to $\sum_{k\in[n]}\delta_k \leq 1$. Thus, in this case, 
	\begin{align}
	\mathbb{E}{N}_\mathcal{A}(\mathcal{I}) = & \tilde{\Omega}\big(\sum_{i\in[n]}\tilde{\Delta}^{-2}_{i}\log\log\tilde{\Delta}_i^{-1}\big) + \Omega\big(\sum_{i\in[n]}{\tilde{\Delta}_i^{-2}}\log{n} + \sum_{i\in[n]}{\tilde{\Delta}_i^{-2}}\log\delta^{-1}\big) \nonumber \\
	= & \tilde{\Omega}\big(\sum_{i\in[n]}\tilde{\Delta}^{-2}_{i}\log\log\tilde{\Delta}_i^{-1}\big) + \Omega\big(\sum_{i\in[n]}{\tilde{\Delta}_i^{-2}}\log(n/\delta)\big),\nonumber
	\end{align}
	which is the lower bound in (\ref{Eq:LB}). This completes the proof of (\ref{Eq:LB}) and Theorem~\ref{Theorem:LB}.
\end{proof}

\subsection{Proof of Theorem~\ref{Theorem:LBMNL}}

\RestateLBMNL*

\begin{proof}
	We prove this theorem by Lemmas~\ref{Lemma:HeadsLowerBound}, \ref{Lemma:LBCoins} and \ref{Lemma:MNLReduction}, which could be of independent interest. The proofs of these three lemmas can be found in Sections~\ref{Sec:HLB}, \ref{Sec:LBC}, and \ref{Sec:MNLReduction}
	
	Suppose that there are two coins with unknown \textit{head probabilities} (the probability that a toss produces a head) $\lambda$ and $\mu$, respectively, and we want to find the more biased one (i.e., the one with the larger head probability). Lemma~\ref{Lemma:HeadsLowerBound} states a lower bound on the number of heads or tails generated for finding the more biased coin, which works even if $\lambda$ and $\mu$ go to $0$. 
	This is in contrast to the lower bounds on the number of tosses given by previous works \cite{LIL2014,LowerBound2012,FKLowerBound2004}, which go to infinity as $\lambda$ and $\mu$ go to 0.
	
	\RestateHeadsLowerBound*
	
	Now we consider $n$ coins $C_1,C_2,...,C_n$ with mean rewards $\mu_1,\mu_2,...,\mu_n$, respectively, where for any $i\in[n]$, $\theta_i/\mu_i=c$ for some constant $c>0$. Define the gaps of coins $\Delta^c_{i,j} := |\mu_i/(\mu_i + \mu_j)-1/2|$, and $\Delta^c_i := \min_{j\neq i}\Delta^c_{i,j}$. We can check that for all $i$ and $j$, $\Delta^c_{i,j} = \Delta_{i,j}$, and $\Delta_i = \tilde{\Delta}_i = \Delta^c_{i}$.
	
	\RestateLBCoins*
	
	The next lemma shows that any algorithm solves a ranking problem under the MNL model can be transformed to solve the pure exploration multi-armed bandit (PEMAB) problem with Bernoulli rewards. Previous works \cite{LimitedRounds2017,ActiveRanking2019,ApproximateRanking2018} have shown that certain types of pairwise ranking problems (e.g., Borda-Score ranking) can also be transformed to PEMAB problems.
	But in this paper, we make a {\em reverse connection} that bridges these two classes of problems, which may be of independent interest.
	
	\RestateMNLReduction*
	
	Combining Lemmas~\ref{Lemma:LBCoins} and \ref{Lemma:MNLReduction}, we have that $\mathbb{E}[N_\mathcal{A}]$ is lower bounded by Eq.~(\ref{Eq:GLB}) with a different hidden constant factor. Then, by the same steps as the Step 6 of the proof of Theorem~\ref{Theorem:LB}, we have that when $\delta \preceq 1/poly(n)$ or $\max_{i,j\in[n]}\{\Delta_i/\Delta_j\} \preceq n^{1/2-p}$ for some constant $p>0$, $\mathbb{E}[N_\mathcal{A}]$ is lower bounded by Eq. (\ref{Eq:LB}) with a different hidden constant factor. This completes the proof. We omit the repetition for brevity and note that under the pairwise MNL model, $\Delta_i = \tilde{\Delta}$ for any item $i$, as the pairwise MNL model satisfies the SST condition.
\end{proof}

\subsection{Proof of Lemma~\ref{Lemma:HeadsLowerBound}}\label{Sec:HLB}

\RestateHeadsLowerBound*

\begin{proof}
	By contradiction, suppose that there is an algorithm $\mathcal{A}$ that does not satisfy the stated lower bound. We will show a contradiction to Lemma~\ref{Lemma:LBTwo}. 
	
	Given a coin with head probability $p=1/2+\eta$, where $\eta\in (-1/4,0)\cup(0,1/4)$ is unknown, we will use $\mathcal{A}$ to construct an algorithm to recover the value of $\sign(\eta)$, i.e. the sign of $\eta$. Choose an $\alpha \in (0,1)$. We recall that a $p$-coin denotes a coin such that each toss of it produces a head with probability $p$, and a tail otherwise.
	
	Now, we construct two i.i.d. sequences of random variables: $\{X^t\}_{t=1}^\infty$ and $\{Y^t\}_{t=1}^\infty$. 
	
	Sequence $\{X^t\}_{t=1}^\infty$ is generated as follows: For any $t\in\mathbb{Z}^+$, with probability $\alpha$, we toss the $p$-coin, and assign $X^t=1$ if the toss gives a head, and assign $X^t=0$ otherwise. With probability $1-\alpha$, we assign $X^t=0$. 
	
	Sequence $\{Y^t\}_{t=1}^\infty$ is generated as follows: For any $t\in\mathbb{Z}^+$,  with probability $\alpha$, we toss the $p$-coin, and assign $Y^t=1$ if the toss gives a tail, and assign $Y^t=0$ otherwise. With probability $1-\alpha$, we assign $Y^t=0$.
	
	As a result, $(X^t,t\in\mathbb{Z}^+)$ are i.i.d. Bernoulli$(\lambda)$, and $(Y^t,t\in\mathbb{Z}^+)$ are i.i.d. Bernoulli$(\mu)$, respectively. Thus, we can view that $X^t$'s are generated by a $\lambda$-coin and $Y^t$'s are generated by a $\mu$-coin, where $\lambda = \alpha(1/2+\eta)$ and $\mu = \alpha(1/2-\eta)$. We check that $|\lambda/(\lambda+\mu)-1/2| = \eta$.
	
	Next, we use algorithm $\mathcal{A}$ to find the more biased one of $(X^t,t\in\mathbb{Z}^+)$ and $(Y^t,t\in\mathbb{Z}^+)$. If the result is $X^t$'s, then we decide $\eta>0$, and if the result is $Y^t$'s, then we decide $\eta<0$. According to the assumption, $\mathcal{A}$ finds the results with probability at least $1-\eta$ and the number of times $t$ such that $X^t=1$ or $Y^t=1$ is at most $o({\eta^{-2}}(\log\log{\eta^{-1}}+\log{\delta^{-1}}))$ in expectation. For each $t$ with $X^t=1$ or $Y^t=1$, the $p$-coin is tossed for at most $4$ times in expectation (since $1/4< p< 3/4$). 
	
	Thus, we can determine whether $\eta < 0$ or $\eta>0$ (equivalent to ranking two items $i$ and $j$ with $p_{i,j}=1/2+\eta$) by $o({\eta^{-2}}(\log\log{\eta^{-1}}+\log{\delta^{-1}}))$ tosses in expectation, contradicting Lemma~\ref{Lemma:LBTwo}. Thus, such an algorithm $\mathcal{A}$ does not exist. This completes the proof of Lemma~\ref{Lemma:HeadsLowerBound}.
\end{proof}

\subsection{Proof of Lemma~\ref{Lemma:LBCoins}}\label{Sec:LBC}

\RestateLBCoins*

\begin{proof}
	To prove this lemma, we need to show the following lower bound:
	\begin{align}
	\tilde{\Omega}\Big(\sum_{i\in[n]}\tilde{\Delta}_i^{-2}\log\log\tilde{\Delta}_i^{-1}\Big) \nonumber + \Omega\Big(\sum_{i\in[n]}\tilde{\Delta}_i^{-2}\log(1/\delta) + \min\{\sum_{i\in[n]}\tilde{\Delta}_i^{-2}\log(1/x_i): \sum_{i\in[n]}x_i \leq 1\}\Big).\nonumber
	\end{align}
	
	The proof is similar to that of Lemma~\ref{Lemma:LBP2}. We assume that the true order of these coins is $(q_1,q_2,...,q_n)$, and $n=2m+1$ is odd. When $n$ is even, we can prove the results in similar steps.
	
	To arrange the coins in the ascending order of head probabilities, one at least needs to distinguish the orders of the pairs $(q_1,q_2), (q_3,q_4),...,(q_{2m-1},q_{2m})$. For any $k$ in $[m]$, to order $q_{2k-1}$ and $q_{2k}$ with probability $1-\delta_k$, by Lemma~\ref{Lemma:HeadsLowerBound}, any $\delta$-correct algorithm generates $\tilde{\Omega}(\Delta_{q_{2k-1},q_{2k}}^{-2}\log\log\Delta_{q_{2k-1},q_{2k}}^{-1}) + \Omega(\Delta_{q_{2k-1},q_{2k}}^{-2}\log\delta_k^{-1}))$ heads in expectation. Thus, by the same steps as in the proof of Lemma~\ref{Lemma:LBP2}, we obtain a lower bound as follows:
	\begin{align}
	&\tilde{\Omega}\Big( \sum_{k\in[m]}\Delta_{q_{2k-1},q_{2k}}^{-2}\log\log\Delta_{q_{2k-1},q_{2k}}^{-1}\Big) \nonumber \\
	& + \Omega\Big(\min\{\sum_{k\in[m]}{\Delta_{q_{2k-1},q_{2k}}^{-2}\log\delta_k^{-1}}: \sum_{k\in[m]}{\delta_k} \leq 2\delta\}\Big).\nonumber
	\end{align}
	
	Also, to find to orders of the pairs $(q_2,q_3),(q_4,q_5),...(q_{2m},q_{2m+1})$, there is another lower bound shown below:
	\begin{align}
	&\tilde{\Omega}\Big( \sum_{k\in[m]}\Delta_{q_{2k},q_{2k+1}}^{-2}\log\log\Delta_{q_{2k},q_{2k+1}}^{-1}\Big) \nonumber \\
	& + \Omega\Big(\min\{\sum_{k\in[m]}{\Delta_{q_{2k},q_{2k+1}}^{-2}\log\delta_k^{-1}}: \sum_{k\in[m]}{\delta_k} \leq 2\delta\}\Big).\nonumber
	\end{align}
	
	By the same steps as the Step 5 of the proof of Theorem~\ref{Theorem:LB}, we can get the desired lower bound. We omit the repetition for brevity. This completes the proof.
\end{proof}

\subsection{Proof of Lemma~\ref{Lemma:MNLReduction}}\label{Sec:MNLReduction}

\RestateMNLReduction*

\begin{proof}
	To prove this lemma, consider the following procedure $\mathcal{A}_c$.
	
	\renewcommand{\thealgorithm}{}
	\begin{algorithm}[h]
		\caption{Procedure $\mathcal{A}_c$}
		\textbf{Input:} Two coins $C_i$ and $C_j$ with unknown head probabilities $\mu_i$ and $\mu_j$, respectively;
		\begin{algorithmic}[1]
			\Repeat 
			\State Randomly choose a coin $C_w$ and toss it;
			\State Let $s\gets 1$ if the the toss gives a head, and $s\gets 0$ otherwise;
			\Until{$s=1$}
			\State \Return $C_w$;
		\end{algorithmic}
	\end{algorithm}
	
	\begin{claim}\label{Claim:ReductionProcedure}
		Procedure $\mathcal{A}_c$ returns coin $C_i$ with probability ${\mu_{i}}/(\mu_i + \mu_j)$ and returns $C_j$ otherwise. 
	\end{claim}
	\begin{proof}[Proof of Claim~\ref{Claim:ReductionProcedure}.]
		Let $T$ be the number of tosses conducted before $\mathcal{A}_c$ returns, and $X$ be the coin it returns. By using conditional probability, we have that for all $t\geq 1$ and $i$ in $[m]$,
		\begin{align}
		\mathbb{P}\left\{T = t, X = C_i\right\} = &\prod_{\tau=1}^{t-1}\mathbb{P}\left\{T > \tau \mid T > \tau - 1\right\} \cdot \mathbb{P}\left\{T = t, X = C_i \mid T > t - 1\right\} \nonumber\\
		& = (\mathbb{P}\left\{T > 1\right\})^{t-1} \cdot \mathbb{P}\left\{T = 1, X = C_i\right\}\nonumber \\
		& = \left(1 - \frac{1}{2}(\mu_i + \mu_j)\right)^{t-1}\cdot\frac{1}{2}\mu_{i},\nonumber
		\end{align}
		\begin{align}
		\mathbb{P}\left\{X = C_i\right\} & = \sum_{t=1}^\infty{\mathbb{P}\left\{T = t, X = C_i\right\}} \nonumber\\
		& = \sum_{t=1}^\infty{\left(1 - \frac{1}{2}(\mu_i + \mu_j)\right)^{t-1}\cdot\frac{1}{2}\mu_{i}}= \frac{\mu_{i}}{\mu_i + \mu_j},\nonumber
		\end{align}
		and the proof of Claim~\ref{Claim:ReductionProcedure} is complete.
	\end{proof}
	
	By Claim~\ref{Claim:ReductionProcedure}, we see that the probabilities that $\mathcal{A}_c$ return arms are with the same form as the MNL model. For a ranking algorithm $\mathcal{A}$, we substitute the input with these $n$ arms and use the procedure $\mathcal{A}_c$ to imitate the comparisons. Whenever the algorithm wants a comparison over $C_i$ and $C_j$, we call procedure $\mathcal{A}_c$ with input $C_i$ and $C_j$. If $\mathcal{A}_c$ returns $C_i$, then we tell $\mathcal{A}$ that $C_i$ wins the comparison, and otherwise, tell $\mathcal{A}_c$ that $C_j$ wins the comparison. Since $\mathcal{A}_c$ returns the arms with probabilities with the same form as the MNL model, $\mathcal{A}$ does not notice any abnormal and work as usual. 
	
	For each call of $\mathcal{A}_c$, there is exactly one head generated. Thus, by this modification, $\mathcal{A}$ arranges these $[n]$ coins in the order of ascending head probabilities with confidence $1-\delta$, and generates $M$ heads in expectation. 
	
	This completes the proof of Lemma~\ref{Lemma:MNLReduction}.
\end{proof}

\subsection{Proof of Proposition~\ref{Proposition:HPC}}

\RestatePropositionHPC*

\begin{proof}
	
	\textbf{Lower Bound.} The proof of the lower bound leverages techniques from information theory. Let $X,Y$ be two discrete random variables (i.e., with at most countably infinite choices of values), and $\Omega_X,\Omega_Y$ be their sample spaces, respectively. We first briefly introduce some terms of information theory. More information about the information theory can be found in standard texts (e.g., \cite{InformationTheory1991}).
	
	Define 
	\begin{gather}
	p_x:=\mathbb{P}\{X=x\},\quad p_y:=\mathbb{P}\{Y=y\},\nonumber\\ p_{x,y}:=\mathbb{P}\{X=x,Y=y\},\quad p_{x\mid y}:=\mathbb{P}\{X=x\mid Y=y\}.\nonumber
	\end{gather} 
	
	The information entropy of $X$ is  defined as
	\begin{align}
	H(X) := \sum_{x\in\Omega_X}p_x\log(1/p_x),\nonumber
	\end{align} and the information entropy of $Y$ is defined as
	\begin{align}
	H(Y) := \sum_{y\in\Omega_Y}p_y\log(1/p_y)\nonumber.
	\end{align} 
	The joint entropy of $X$ and $Y$ is 
	\begin{align}
	H(X,Y) :=\sum_{x\in\Omega_X,y\in\Omega_Y}p_{x,y}\log(1/p_{x,y}).\nonumber
	\end{align} 
	The conditional entropy of $X$ given $Y=y$ is 
	\begin{align}
	H(X\mid Y=y) := \sum_{x\in\Omega_X}p_{x\mid y}\log(1/p_{x\mid y}),\nonumber
	\end{align} 
	and the conditional entropy of $X$ given $Y$ is 
	\begin{align}
	H(X\mid Y)=\sum_{y\in\Omega_Y}p_yH(X\mid Y=y).\nonumber
	\end{align} 
	The mutual information of $X$ and $Y$ is 
	\begin{align}
	I(X;Y)=\sum_{x\in\Omega_X,y\in\Omega}p_{x,y}\log\frac{p_{x,y}}{p_xp_y}.\nonumber
	\end{align} 
	Given another discrete random variable $Z$, the conditional mutual information of $X$ and $Y$ given $Z$ is 
	\begin{align}
	I(X;Y\mid Z)= I(X;Y,Z)-I(X;Z).\nonumber
	\end{align}
	
	We further have the following facts \cite{InformationTheory1991}
	\begin{gather}
	H(X)\leq \log|\Omega_X|,\nonumber\\ 
	H(X\mid Y)\leq H(X) \leq H(X,Y),\nonumber \\ 
	H(X,Y)=H(Y)+H(X\mid Y)=H(X)+H(Y\mid X),\nonumber \\
	I(X;Y)=H(X)-H(X\mid Y)\nonumber,\\
	I(X;Y\mid Z)\leq I(X;Y).\nonumber
	\end{gather} 
	Also, if $X$ is determined by $Y$, then
	\begin{align}
	H(X\mid Y) = 0. \nonumber
	\end{align}
	
	With the above introduction of information, we show the following fact that is used in the proof. 
	\begin{fact}[Fano's Inequality \cite{Fano1961}]\label{Fact:Fano}
		To recover the value of $X$ from $Y$ with error probability no more than $\delta$, it must hold that  
		\begin{align}
		H(X|Y)\leq H(\delta)+\delta\log(|\Omega_X|-1).\nonumber
		\end{align}
	\end{fact}
	
	The key idea to prove the lower bound is to show that if the expected number of samples conducted is lower than the lower bound, then Fano's Inequality will not be satisfied. 
	
	From now on, we assume that all the comparisons are correct and choose $\delta=1/4$. We reuse some notation and let $X$ be the ranking of the $n$ items. Before any comparison, we have no information about it, and thus, each ranking has the same probability to be the correct one. Since there are $n!$ possible permutations in total, we have that $H(X)=\log(n!) \simeq n\log{n}$. 
	
	Let $\mathcal{A}$ be an algorithm that adaptively selects the sets to compare and determine whether to stop by past comparison outcomes, let $N$ be the number of comparisons conducted till termination (i.e., stopping time). Let $\vec{S}=(S_1,S_2,...,S_N)$ be the sequence of sets that the algorithm compares. Let $\vec{Y}=(Y_1,Y_2,...,Y_N)$ be the sequence of comparison outcomes generated by the algorithm. For any $t$, $S_t$ is of the form $(S_{t}[1],S_{t}[2],...,S_{t}[m])$, which consists of the items compared in the $t$-th comparison. The value of $Y_t$ is in $\{1,2,...,m\}$, where $Y_t=i$ means the winner of the $t$-th comparison is $S_{t}[i]$. We assume that $\mathcal{A}$ is deterministic, i.e., the value of $S_t$ is determined by $(Y_1,Y_2,...,Y_{t-1})$ and $(S_1,S_2,...,S_{t-1})$, and $N$ is determined by $\vec{Y}$ and $\vec{S}$. We have
	\begin{align}
	I(X;\vec{S}|\vec{Y},N)\leq & I(X;\vec{S}\mid \vec{Y}) \nonumber
	\leq   H(\vec{S}|\vec{Y})\nonumber \\
	=& H(S_1\mid\vec{Y}) + H(S_2\mid\vec{Y})  + \cdots H(S_N\mid\vec{Y}) \nonumber \\
	\leq &  H(S_1) + H(S_2\mid Y_1,)  + \cdots H(S_N\mid Y_1,Y_2,...,Y_{N-1}) \nonumber \nonumber \\
	= & 0.\label{Eq:HAY}
	\end{align}
	
	Also, for any $t$-th comparison, there are at most $m$ different choices of values for $Y_t$, and thus, $H(Y_t)\leq \log{m}$. For any $n\in\mathbb{Z}^+$, when $N=n$, the number of choices of values of $\vec{Y}$ is at most $m^n$, so $H(\vec{Y}|N=n)\leq n\log{m}$, which implies that 
	\begin{align}\label{Eq:HYN}
	H(\vec{Y}|N) =  \sum_{n=1}^\infty{\mathbb{P}\left\{N=n\right\}H(\vec{Y}\mid N=n)}\leq \mathbb{E}N\log{m}.
	\end{align}
	
	Now, we bound $H(N)$ by $\mathbb{E}N$. Define a random variable $R$ such that $R=0$ if $N<2\mathbb{E}N$ and $R=k$ if $2^k\mathbb{E}N\leq N<2^{k+1}\mathbb{E}N$ for any $k\in\mathbb{Z}^+$. By Markov's Inequality, we have that for $k\in\mathbb{Z}^+$,
	\begin{align}\label{Eq:PR}
	\mathbb{P}\{R=k\}=\mathbb{P}\{2^k\mathbb{E}N\leq N<2^{k+1}\mathbb{E}N\}\leq \mathbb{P}\{N \geq 2^k\mathbb{E}N\}\leq 2^{-k},
	\end{align}
	
	Use $p_k$ to denote $\mathbb{P}\{R=k\}$. By analyzing the function $p\log(1/p),p\in[0,1]$, it holds that
	\begin{align}\label{Eq:HR}
	H(R)=p_0\log(1/p_0) + \sum_{k=1}^\infty p_k\log(1/p_k) \leq 2/e + \sum_{k=2}^\infty{2^{-k}\log(2^k)}\leq 2/e + (3/2)\log{2}.
	\end{align}
	
	Noting that $H(N\mid N\in S)\leq \log{|S|}$ for all sets $S$, we have
	\begin{align}\label{Eq:HN}
	H(N) = & H(R) + H(N|R) \nonumber \\
	= & H(R) +  \mathbb{P}\left\{N<2\mathbb{E}N\right\}H\left(N\mid N<2\mathbb{E}N\right) \nonumber \\ 
	+ &\sum_{k=1}^\infty\mathbb{P}\left\{2^k\mathbb{E}N\leq N<2^{k+1}\mathbb{E}N\right\}H\left(N\mid 2^k\mathbb{E}N\leq N<2^{k+1}\mathbb{E}N\right)\nonumber \\
	\stackrel{(a)}{\leq} & 2/e + (3/2)\log{2} + \log\left(2\mathbb{E}N\right) + \sum_{i=1}^\infty{2^{-k}\log\left(2^k\mathbb{E}N\right)}\nonumber \\
	\leq & 2/e + \log\left(24\mathbb{E}^2N\right),
	\end{align}
	where (a) is due to (\ref{Eq:PR}) and (\ref{Eq:HR}). 
	
	By (\ref{Eq:HAY}) (\ref{Eq:HYN}) (\ref{Eq:HN}), we have
	\begin{align}\label{Eq:HXY}
	H\left(X\mid N,\vec{Y},\vec{A}\right) = & H(X) - I\left(X; N,\vec{Y},\vec{A}\right) \nonumber \\
	\geq & H(X) - H\left(N,\vec{Y},\vec{A}\right)\nonumber \\
	= & H(X) - \left(H(N) + H(\vec{Y}\mid N) + H(\vec{A}\mid N,\vec{Y})\right)\nonumber \\
	\geq & \log(n!) - \left(2/e + \log\left(24\mathbb{E}^2N\right) + \mathbb{E}N\log{m} + 0\right).
	\end{align}
	
	By Fano's Inequality, to recover $X$ with probability at least $1/4$, it must hold that 
	\begin{align}
	H(X\mid N,\vec{Y}, \vec{A}) \leq H(1/4) + (1/4)\log(n!-1),\nonumber
	\end{align} 
	which, along with (\ref{Eq:HXY}) and $\log(n!)=\Theta(n\log{n})$, implies that 
	\begin{align}
	\mathbb{E}N = \Omega(n\log_m{n}).\nonumber
	\end{align} 
	For randomized algorithms, its sample complexity is no less than that of the fastest deterministic algorithm, and thus, satisfies the same lower bound. This proves the lower bound. 
	
	\noindent\textbf{Upper Bound.} To see the upper bound, consider the following ListwiseMergeSort (LWMS) algorithm, which is presented in Algorithm~\ref{Alg:LWMS}. LWMS is similar to the binary merge-sort. Algorithm~\ref{Alg:LWM} ListwiseMerge is the subroutine of LWMS, which merges $m$ sorted lists of items.
	
	\algtext{EndIf}{\textbf{end if}}
	
	\begin{algorithm}[h]
		\caption{ListwiseMerge$(A_1,A_2,...,A_m,m)$}\label{Alg:LWM}
		\begin{algorithmic}[1]
			\State $Ans\gets$ an empty list to store the result; 
			\State For all $i$ in $[m]$, Let $I_i\gets 1$ be the index of $A_i$;
			\While{$\exists i\in[m]$,$I_i\leq|A_i|$}
			\State $B\gets\{A_i[I_i]:I_i\leq A_i\}$; 
			\State Conduct a listwise comparison over $B$, and let $A_j[I_j]$ be the winner;
			\State Push $A_j[I_j]$ to the end of $Ans$; $I_j\gets I_j+1$;
			\EndWhile
			\State \Return{$Ans$}
		\end{algorithmic}
	\end{algorithm}
	
	\begin{algorithm}[h]
		\caption{ListwiseMergeSort$(S,m)$ (LWMS)}\label{Alg:LWMS}
		\begin{algorithmic}[1]
			\If{$|S|=1$} 
			\State \Return{$S$;}\quad{\# No need to do anything}
			\EndIf
			\State Divide $S$ into $m$ sets $A_1,A_2,...,A_3$ such that $|A_i|\leq \lceil|S|/m\rceil|$ for all $i\in[m]$;
			\For{$i\in[m]$}
			\State $A_i\gets$ ListwiseMergeSort$(A_i,m)$
			\EndFor
			\State \Return{ListwiseMerge$(A_1,A_2,...,A_m,m)$};
		\end{algorithmic}
	\end{algorithm}
	
	\begin{lemma}[Theoretical upper bound of LWMS]\label{TP-LWMS}
		Algorithm LWMS correctly ranks $n$ items with high probability using $O(n\log_m{n})$ comparisons.
	\end{lemma}
	\begin{proof}
		We use $T_s(x)$ to denote the number of comparisons needed to rank (sort) $x$ items, and use $T_m(x)$ to denote the number of comparisons needed to merge $m$ sorted lists with $x$ items in total. In the algorithm ListwiseMerge, since after each comparison, a new item is added to the result $Ans$, we have that $T_m(x)\leq x$. Also, we have that $T_s(1)=0$, and for all $t\geq 1$, $T_s(m^t)= mT_s(m^{t-1}) + T_m(m^t)$. It then follows that $T_s(m^t)\leq tm^t$, which implies $T_s(n)=O(n\log_m{n})$. This completes the proof.
	\end{proof}	
	
	This completes the proof of Proposition~\ref{Proposition:HPC}.
\end{proof}

\subsection{Proof of Theorem~\ref{Theorem:LBLWMNL}}

\RestateLBLWMNL*

\begin{proof}
	Let $n$ coins $C_1,C_2,...,C_n$ with unknown head probabilities $\mu_1,\mu_2,...,\mu_n$ be given, where $\mu_i/\theta_i$ is a fixed constant for all $i\in[n]$. We only need to show that the reduction from PEMAB problems to exact ranking stated in Lemma~\ref{Lemma:MNLReduction} still holds for listwise comparisons under the MNL model.
	
	Consider the the following procedure:
	\renewcommand{\thealgorithm}{}
	\begin{algorithm}
		\caption{Procedure $\mathcal{A}'_c$}
		\textbf{Input:} Totally $m$ coins $C_{r_1},C_{r_2},...,C_{r_m}$ with unknown head probabilities $\mu_{r_1},\mu_{r_2},...,\mu_{r_m}$;
		\begin{algorithmic}[1]
			\Repeat
			\State Randomly choose a coin $C_w$ and toss it;
			\State Let $s\gets 1$ if the toss gives a head, and let $s\gets 0$ otherwise;
			\Until{$s=1$}
			\State \Return{$C_w$};
		\end{algorithmic}
	\end{algorithm}
\end{proof}

\begin{claim}\label{Claim:LWReductionProcedure}
	Procedure $\mathcal{A}'_c$ returns a coin $C_{r_i}$ with probability ${\mu_{r_i}}/{\sum_{j=1}^m{\mu_{r_j}}}$. 
\end{claim}
\begin{proof}[Proof of Claim~\ref{Claim:LWReductionProcedure}.]
	Let $T$ be the number of tosses conducted before $\mathcal{A}'_c$ returns, and $X$ be the coin $\mathcal{A}'_c$ returns. By using conditional probability, we have that for all $t\geq 1$ and $i$ in $[m]$,
	\begin{align}
	\mathbb{P}\left\{T = t, X = C_{r_i}\right\} = &\prod_{\tau=1}^{t-1}\mathbb{P}\left\{T > \tau \mid T > \tau - 1\right\} \cdot \mathbb{P}\left\{T = t, X = C_{r_i} \mid T > t - 1\right\} \nonumber\\
	& = (\mathbb{P}\left\{T > 1\right\})^{t-1} \mathbb{P}\left\{T = 1, X = C_{r_i}\right\}\nonumber \\
	& = \left(1 - \frac{1}{m}\sum_{j=1}^m{\mu_{r_j}}\right)^{t-1}\cdot \frac{1}{m}\mu_{r_i},\nonumber
	\end{align}
	\begin{align}
	\mathbb{P}\left\{X = C_{r_i}\right\} & = \sum_{t=1}^\infty{\mathbb{P}\left\{T = t, X = C_{r_i}\right\}} \nonumber\\
	& = \sum_{t=1}^\infty{\left(1 - \frac{1}{m}\cdot \sum_{j=1}^m{\mu_{r_j}}\right)^{t-1}}\cdot\frac{1}{m}\mu_{r_i}= \frac{\mu_{r_i}}{\sum_{j=1}^m{\mu_{r_j}}},\nonumber
	\end{align}
	and the proof of the claim is complete.
\end{proof}

The proof of Theorem~\ref{Theorem:LBLWMNL} is complete by Lemma~\ref{Lemma:LBCoins} and the same steps as in the proof of Theorem~\ref{Theorem:LBMNL}, the pairwise lower bound for the MNL model.

\subsection{Proof of Lemma~\ref{Lemma:TP-ATC}}

\RestateTPATC*

\begin{proof}
	Without loss of generality, we assume $i\succ j$. Since the for loop runs at most $b^{max}=\lceil\frac{1}{2}\epsilon^{-2}\log(2\delta^{-1})\rceil$ iterations and each iteration performs one comparison, the subroutine returns after at most $O(\epsilon^{-2}\log\delta^{-1})$ comparisons. Since the return condition of items $i$ and $j$ are symmetric and $i\succ j$, by this symmetry, ATC returns $j$ with probability no more than $1/2$. 
	
	Now we consider the case where $p_{i,j}\geq 1/2+\epsilon$, and it remains to prove that ATC returns $i$ with probability at least $1-\delta$. Define $b^t := \sqrt{\frac{1}{2t}\log\frac{\pi^2 t^2}{3\delta}}$. Let $\mathcal{E}^{out}_t$ be the event that $\hat{p}^t_i\leq p_{i,j}-b^t$, and define $\mathcal{E}^{out}:=\bigcup_{t=1}^\infty\mathcal{E}^{out}_t$. We have
	\begin{align}
	\mathbb{P}\left\{\mathcal{E}^{out}\right\} \stackrel{(a)}{\leq} \sum_{t=1}^\infty\mathbb{P}\left\{\mathcal{E}^{out}_t\right\}\stackrel{(b)}{\leq} \sum_{t=1}^\infty\left[\exp\left(-2t\left(b^t\right)^2\right)\right]\leq \sum_{t=1}^\infty{\frac{3\delta}{\pi^2 t^2}}\leq \frac{\delta}{2},\label{Eq:Eout}
	\end{align}
	where (a) is due to the union bound and (b) is due to the Chernoff-Hoeffding Inequality \cite{Hoeffding1963}. 
	
	Assume that $\mathcal{E}^{out}$ does not happen, and we have that for all $t$, $\hat{p}^t_i > 1/2+\epsilon-b^t\geq 1/2-b^t$. Thus, ATC does not return $j$ during the for loop with probability at least $1-\delta/2$.
	
	After the for loop, by Chernoff-Hoeffding Inequality and $b^{max} = \lceil\frac{1}{2\epsilon^2}\log\frac{2}{\delta}\rceil$, we have
	\begin{align}
	\mathbb{P}\left\{\hat{p}^{b^{max}}_i\leq 1/2\right\} \leq \exp\left\{-2b^{max}(p_{i,j}-1/2)^2\right\} \leq \exp\left\{-2b^{max}\epsilon^2\right\} \leq {\delta}/{2},
	\end{align}
	which implies that the last line of ATC returns $i$ with probability at least $1-{\delta}/{2}$. This completes the proof of Lemma~\ref{Lemma:TP-ATC}.
\end{proof} 

\subsection{Proof of Lemma~\ref{Lemma:TP-ATI}}

\RestateTPATI*

\begin{proof}
	\textbf{(I)} We first prove the sample complexity. We observe that for a constant $\delta_0\in(0,1/2)$, a call of ATC$(i,j,\epsilon,\delta_0)$ returns after at most $O({\epsilon^{-2}})$ comparisons by Lemma~\ref{Lemma:TP-ATC}. In ATI, for each iteration, there are at most three calls of ATC and all the calls are with constant confidence. Also, ATI returns after at most $t^{max}=O(h+\log{\delta^{-1}})$ iterations, where $h=1+\lceil\log_2(1+|S|)\rceil=O(\log{|S|})$. Thus, the number of comparisons is at most $3t^{max}\cdot O(\epsilon^{-2}) = O(\epsilon^{-2}\log(|S|/\delta))$. This completes the proof sample complexity.
	
	\textbf{(II)} We prove that ATI does not insert $i$ into a wrong place with probability at least $1/2$. A round (or iteration) is said to be \textit{correct} if during this round, all calls of ATC return the more preferred item, and is said to be \textit{incorrect} otherwise. A leaf node $u$ is said to be \textit{correct} if $i\in (u.\mbox{left},u.\mbox{right})$, i.e., $i$ belongs to the corresponding interval of$u$. A leaf node $u$ is said to be \textit{incorrect} if it is not correct.
	
	For any round $t$, we define an event $\mathcal{E}^t_{il}$ such that
	\begin{align}
	\mathcal{E}^t_{il} := \{X=\mbox{some incorrect leaf node at the beginning of round } t\mbox{ and }c_X \geq 1\}.
	\end{align}
	
	We assume that for some round $t$, $\mathcal{E}^t_{il}$ happens, which implies that $i\succ u\mbox{.right}$ or $u\mbox{.left}\succ i$, i.e., $i$ does not belong to the interval of $u$. By Lemma~\ref{Lemma:TP-ATC} the property of ATC, it holds that 
	\begin{gather}
	\mathbb{P}\left\{\mbox{ATC}(i,u\mbox{.right},\epsilon,\delta) = i\mid i\succ u\mbox{.right}\right\} \geq 1/2,\nonumber \\
	\mathbb{P}\left\{\mbox{ATC}(i,u\mbox{.left},\epsilon,\delta) = u\mbox{.left} \mid u\mbox{.left} \succ i\right\} \geq 1/2.\nonumber 
	\end{gather}
	which implies that for any round $t$,
	\begin{align}\label{Eq:wrongInterval}
	\mathbb{P}\left\{\mbox{round }t\mbox{ is correct}\mid \mathcal{E}^t_{il}\right\} \geq 1/2.
	\end{align}
	
	For any $t$, define 
	\begin{gather}
	R_1^t := \left|\left\{\tau\leq t: \mbox{round } \tau\mbox{ is correct, and } \mathcal{E}^\tau_{il}\mbox{ happens}\right\}\right| \nonumber , \\
	W_1^t := \left|\left\{\tau\leq t: \mbox{round } \tau\mbox{ is incorrect, and } \mathcal{E}^\tau_{il}\mbox{ happens}\right\}\right| \nonumber .
	\end{gather}
	
	For any incorrect leaf node $u$ and any round $t$, the counter $c_u$ is increased by one during this round if and only if $\mathcal{E}^\tau_{il}$ happens and this round is incorrect. Also, for any round $t$, given $\mathcal{E}^\tau_{il}$, if this round is correct, then the counter $c_u$ is decreased by one. Thus, for any incorrect leaf node $u$, at the end of any round $t$, the value of $c_u$ is at most 
	\begin{align}
	c_u(t) \leq 1 + W_1^t - R_1^t.\nonumber
	\end{align}
	
	After the for loop, ATI incorrectly inserts $i$ if and only if some incorrect leaf node $u$ is counted for $\frac{5}{16}t^{max}+1$ times, i.e., $c_u\geq \frac{5}{16}t^{max}+1$, which implies $W_1^{t^{max}}-R_1^{t^{max}}\geq \frac{5}{16}t^{max}$. Thus, by the fact that $W_1^{t^{max}}+R_1^{t^{max}} \leq t^{max}$, and Eq.~\ref{Eq:wrongInterval}, we obtain
	\begin{align}
	&\mathbb{P}\left\{W_1^{t^{max}}-R_1^{t^{max}}\geq \frac{5}{16}t^{max}\right\} \nonumber \\
	= & \mathbb{P}\left\{W_1^{t^{max}}\geq \frac{1}{2}\left(R_1^{t^{max}}+W_1^{t^{max}}+\frac{5}{16}t^{max}\right)\right\}\nonumber \\
	\stackrel{(a)}{\leq} & \sup_{K\leq t^{max}}\mathbb{P}\left\{\frac{W_1^{t^{max}}}{W_1^{t^{max}}+R_1^{t^{max}}}\geq \frac{1}{2}+\frac{5}{32}\cdot\frac{t^{max}}{W_1^{t^{max}}+R_1^{t^{max}}} \Big| W_1^{t^{max}}+R_1^{t^{max}} = K\right\}\nonumber \\
	\stackrel{(b)}{\leq} & \sup_{K \leq t^{max}}\exp\left\{-2K\left(\frac{5t^{max}}{32K}\right)^2\right\} \nonumber \\
	= & 
	\exp\left\{-2t^{max}\left(\frac{5t^{max}}{32t^{max}}\right)^2\right\} \leq {\delta}/{2},\label{Eq:AfterFor}
	\end{align}
	where (a) is due to $R_1^{t^{max}}+W_1^{t^{max}}\leq t^{max}$, and (b) follows from Chernoff-Hoeffding Inequality. This proves that with probability at least $1-\delta/2$, $i$ is not inserted into a wrong place by the second last line. 
	
	Then, during the for loop, for any $t\leq t^{max}$, by (\ref{Eq:wrongInterval}) and Chernoff-Hoeffding Inequality, we have that at the end of the $t$-th round, the probability that $X$ equals to an incorrect leaf node and $c_X>\frac{1}{2}t +\sqrt{\frac{t}{2}\log\frac{\pi^2t^2}{3\delta}} + 1$ is at most 
	\begin{align}
	\mathbb{P}\left\{ W_1^t - R_1^t \geq \frac{1}{2}t + \sqrt{\frac{t}{2}\log\frac{\pi^2t^2}{3\delta}} \right\} \leq & \mathbb{P}\left\{W_1^t \geq \frac{1}{2}t + \sqrt{\frac{t}{2}\log\frac{\pi^2t^2}{3\delta}} \right\} \nonumber \\
	\leq & \exp\left\{-\frac{2}{t}\left(\frac{1}{2}t+\sqrt{\frac{t}{2}\log\frac{\pi^2t^2}{3\delta}}-\frac{t}{2}\right)^2\right\}\leq \frac{3\delta}{\pi^2t^2}. \nonumber
	\end{align}
	
	Since 
	\begin{align}
	\sum_{t=1}^\infty\frac{3\delta}{\pi^2t^2}\leq {\delta}/{2}, \nonumber
	\end{align}
	during the for loop, with probability at least $1-\delta/2$, ATI does not insert $i$ into a wrong place. This, along with Eq.~(\ref{Eq:AfterFor}), proves that with probability at least $1-\delta$, ATI does not insert $i$ into a wrong place. This completes the proof of the first part of Lemma~\ref{Lemma:TP-ATI}.
	
	\textbf{(III)} In this part, we assume $\epsilon \leq \Delta_{i}$ and we prove the second part of Lemma~\ref{Lemma:TP-ATI}. For any round $t$,  by Lemma~\ref{Lemma:TP-ATC} and the choice of input parameters of the calls of ATC, this round is correct with probability at least $q$. Here, we define $R$ as the number of correct rounds before termination, and let $W$ be the number of incorrect rounds before termination. 
	
	Let $u_0$ be the correct node. Define the distance between two nodes $u$ and $v$ as $d(u,v) :=$ the length of the shortest path from $u$ to the $v$, i.e., the number of edges between $u$ and $v$. During each correct round, either $d(X,u_0)$ is decreased by one or the value of $c_{u_0}$ is increased by one, i.e., $c_{u_0}-d(X,u_0)$ is increased by one. During each incorrect round, either $d(X,u_0)$ is increased by one or the value of $c_{u_0}$ is decreased by one, i.e., $c_{u_0}-d(X,u_0)$ is decreased by one. Since the distance between the start node (i.e., the root node) and $u_0$ is at most $h-1$, we always have 
	\begin{align}
	R - W \leq h - 1 + ( c_{u_0} - d(X,u_0)). \nonumber
	\end{align} After the for loop, if $c_{u_0}\geq  \frac{5}{16}t^{max}+1$, then ATI correctly inserts $i$. Thus, if $R-W \geq h + \frac{5}{16}t^{max}$, then ATI correctly inserts $i$. 
	
	Assume that ATI does not return during the for loop, and then, we have $R+W=t^{max}$. For all $t$, round $t$ is correct with probability at least $q$ by Lemma~\ref{Lemma:TP-ATC} and the choices of input parameters of the calls of ATC, hence, by $t^{max}\geq \max\{4h, \frac{512}{25}\log\frac{2}{\delta}\}$ and $q = 15/16$, we have
	\begin{align}
	\mathbb{P}\left\{R-W < h+\frac{5}{16}t^{max}\right\} \stackrel{(a)}{\leq}& \mathbb{P}\left\{R-W < \left(\frac{1}{4}+\frac{5}{16}\right)t^{max}\right\} \nonumber \\
	= & \mathbb{P}\left\{R-(t^{max}-R) < \left(\frac{1}{4}+\frac{5}{16}\right)t^{max}\right\} \nonumber\\
	= & \mathbb{P}\left\{R< \frac{25}{32}t^{max}\right\}\nonumber \\
	\stackrel{(b)}{\leq} & \exp\left\{-2t^{max}\left(q-\frac{25}{32}\right)^2\right\} \leq \frac{\delta}{2}, \nonumber
	\end{align}
	where (a) is due to $t^{max}\geq 4h$ and (b) follows from Chernoff-Hoeffding Inequality. 
	
	In conclusion, when $\epsilon \leq \Delta_{i}$, if ATI does not return during the for loop, then it will, with probability at least $1-\delta/2$, insert $i$ into a correct position by the second last line (after the for loop). Also, by part (II), with probability at least $1-\delta/2$, ATI does not insert $i$ into a wrong position during the for loop. Thus, when $\epsilon \leq \Delta_i$, ATI correctly inserts the input item $i$ with probability at least $1-\delta$. This proves the second part of Lemma~\ref{Lemma:TP-ATI}, and along with parts (I) and (II), completes the proof.
\end{proof}

\subsection{Proof of Lemma~\ref{Lemma:TP-IAI}}

\RestateTPIAI*

\begin{proof}
	Define events
	\begin{align}
	\mathcal{E}_1^t := & \{\epsilon_t > \Delta_i\mbox{ and } \mbox{IAI does not insert }i \mbox{ into a wrong position}\}, \nonumber \\
	\mathcal{E}_2^t := & \{\epsilon_t \leq \Delta_i\mbox{ and } \mbox{IAI correctly inserts } i\}, \nonumber 
	\end{align}
	and the bad event
	\begin{align}
	\mathcal{E}^{bad} := \bigcup_{t=1}^\infty(\mathcal{E}_1^t\cup\mathcal{E}_2^t)^\complement. \nonumber 
	\end{align}
	By the union bound and Lemma~\ref{Lemma:TP-ATI}, we have
	\begin{align}
	\mathbb{P}\{\mathcal{E}^{bad}\} \leq \sum_{t=1}^\infty\mathbb{P}\left\{\left(\mathcal{E}_1^t\cup\mathcal{E}_2^t\right)^\complement\right\} \leq \sum_{t=1}^\infty\delta_t = \sum_{t=1}^\infty{\frac{6\delta}{\pi^2 t^2}} = \delta. \nonumber
	\end{align}
	
	In this proof, we assume that $\mathcal{E}^{bad}$ does not happen. 
	
	\textbf{Correctness.} We first prove the correctness. By the definition of $\mathcal{E}^{bad}$, for all $t$ such that $\epsilon_t > \Delta_i$, IAI does not insert $i$ into a wrong position, and when $\epsilon_t \leq \Delta_i$, IAI correctly inserts $i$. Since $\lim_{t\rightarrow\infty}\epsilon_t=0$, there is a $t^*$ such that $\epsilon_{t^*}\leq \Delta_i$. Thus, when $\mathcal{E}^{bad}$ does not happen, IAI correctly inserts $i$. Since $\mathcal{E}^{bad}$ happens with probability at most $\delta$, the correctness follows. 
	
	\textbf{Sample complexity}. Second, we prove the sample complexity. Let $\tau$ be the integer such that $\epsilon_\tau\leq \Delta_i<\epsilon_{\tau-1}$. By the definition of $\mathcal{E}^{bad}$, when $\mathcal{E}^{bad}$ does not happen, IAI correctly inserts $i$ and returns before the end of the $\tau$-th round. 
	
	By $\epsilon_{\tau-1} = 2^{-\tau}$ and $\epsilon_{\tau-1} > \Delta_i$, we have $\tau<\log_2{\Delta_i^{-1}}$. For $1\leq t\leq \tau$, by Lemma~\ref{Lemma:TP-ATI}, the $t$-th round of IAI conducts at most $O({\epsilon_t^{-2}}\log({|S|}\cdot{\delta_t^{-1}}))$ comparisons. Thus, given $\mathcal{E}^{bad}$ does not happen, the number of comparisons conducted by IAI is at most 
	\begin{align}
	O\Big(\sum_{t=1}^\tau{{\epsilon_t^{-2}}\log\left(|S|/\delta_t\right)}\Big) \stackrel{(a)}{=} & O\Big(\sum_{t=1}^\tau{\left(2^{t+1}\right)^2}\log\left(\pi^2t^2|S|/(6\delta)\right)\Big) \nonumber \\
	= &  O\Big(\sum_{t=1}^\tau{4^{t}}\cdot \log\left(|S|\tau /\delta\right)\Big) \nonumber \\
	= &  O({4^{\tau}}\cdot \log\left(|S|\tau/\delta\right)) \nonumber \\
	\stackrel{(b)}{=} & O\left(4^{\log_2{(1/\Delta_i)}}\cdot \log(|S|\cdot\log(1/\Delta_i) / \delta)\right) \nonumber \\
	= & O\left({\Delta_i^{-2}}\left(\log\log{\Delta_i^{-1}}+\log\left({|S|}/{\delta}\right)\right)\right), \nonumber
	\end{align}
	where (a) follows from $\epsilon_t = 2^{t+1}$ and $\delta_t = \frac{6\delta}{\pi^2 t^2}$, and (b) is due to $\tau<\log_2(1/\Delta_i)$. 
	This proves the sample complexity.
	
	The proof of Lemma~\ref{Lemma:TP-IAI} is complete.
\end{proof}

\subsection{Proof of Theorem~\ref{Theorem:TP-IIR}}

\RestateTPIIR*

\begin{proof}
	At iteration $t$ for each $t\in\{2,3,...,n\}$, by Lemma~\ref{Lemma:TP-IAI}, with probability at least $1-\delta/(n-1)$, the call of IAI correctly inserts $S[t]$ into $Ans$, and uses at most $O(\Delta_{S[t]}^{-2}(\log\log\Delta_{S[t]}^{-1}+\log(n/\delta)))$ comparisons. The desired sample complexity follows by summing up the upper bounds for $t\in\{2,3,...,n\}$. For correctness, if all calls of IAI are correct (which happens with probability at least $1-\delta$ by the union bound), then IIR correctly returns the true ranking. This completes the proof.
\end{proof}

\subsection{Proof of Lemma~\ref{Lemma:LBTwo}}\label{Sec:LB2}

\RestateLBTwo*


\begin{proof}
	We will invoke the results for pure exploration multi-armed bandit (PEMAB) problems, and we refer to \cite{FKLowerBound2004} as a reference for details about PEMAB. Assume that there is an arm $a$, and whenever it is pulled for the $t$-th time, it gives an i.i.d. reward $Y^t$. Further assume that for $t\in \mathbb{Z}^+$, $Y^t$ is a Gaussian random variable with mean $\eta$ and variance $1$. We assume that $|\eta|\leq 1/2$ and $\eta\neq 0$. Let $\mathcal{B}$ be a $\delta$-correct algorithm that has no knowledge of $\eta$ and is able to tell whether $\eta>0$ with probability $1-\delta$ for any non-zero $\eta$-value. Let $T_\mathcal{B}(\eta)$ be the number of pulls $\mathcal{B}$ uses before termination under the given $\eta$-value. The authors of \cite{LIL2014,RatioTest1964} have shown that 
	\begin{align}\label{Eq:GaussianLB}
	\limsup_{|\eta|\rightarrow 0}\frac{\mathbb{E}[T_\mathcal{B}(\eta)]}{\eta^{-2}\log\log\eta^{-2}} \geq 2 - 4\delta.
	\end{align}
	
	In this proof, we reduce the problem of distinguishing whether $\eta>0$ to the problem of ranking two items. For any $t\in\mathbb{Z}^+$, if $0 < \eta < 1/2$, we have
	\begin{align}
	\mathbb{P}\{Y^t \geq 0\} = \frac{1}{\sqrt{2\pi}}\int_{-\eta}^{\infty}e^{-\frac{x^2}{2}}\mathrm{d}x \geq \frac{1}{2} + \frac{\eta}{\sqrt{2\pi}}\cdot e^{-\frac{\eta^2}{2}} \geq \frac{1}{2} + \frac{\eta}{\sqrt{2\pi}} \cdot e^{-1/8},\nonumber
	\end{align}
	and if $-1/2 < \eta <0$, we have
	\begin{align}
	\mathbb{P}\{Y^t < 0\} = \frac{1}{\sqrt{2\pi}} \int_{-\infty}^{-\eta}e^{-\frac{x^2}{2}}\mathrm{d}x \geq \frac{1}{2} + \frac{|\eta|}{\sqrt{2\pi}} \cdot e^{-\frac{\eta^2}{2}} \geq \frac{1}{2} + \frac{|\eta|}{\sqrt{2\pi}}\cdot e^{-1/8}. \nonumber
	\end{align}
	
	For each $t$, we let $Z^t = 2\cdot\mathds{1}\{Y^t \geq 0\}-1$. When $\eta>0$, $Z^t$ is with probability at least $\frac{1}{2} + \frac{\eta}{\sqrt{2\pi}}\cdot e^{-1/8}$ to be $1$, and when $\eta<0$, it is with probability at least $\frac{1}{2} + \frac{|\eta|}{\sqrt{2\pi}}\cdot e^{-1/8}$ to be $-1$. Thus, we can view that $(Z^t,t\in\mathbb{Z}^+)$ are generated by tossing a coin with $\mathbb{P}\{Y^t\geq 0\}$ head probability, and we have $|\mathbb{P}\{Y^t\geq 0\}-1/2| \geq \frac{|\eta|}{\sqrt{2\pi}}\cdot e^{-1/8}$. Assume $\mathcal{A}_2$ can ranking two items $i$ and $j$ with probability $1-\delta$ by $T_{\mathcal{A}_2}(\Delta_{i,j})$ expected number of comparisons, then it can find whether $\eta > 0$ by at most $T_{\mathcal{A}_2}(\frac{\eta}{\sqrt{2\pi}}\cdot e^{-1/8})$ expected number of pulls of the arm $a$. Thus, by (\ref{Eq:GaussianLB}), we have
	\begin{align}\label{Eq:LogLogTwo-o2}
		\limsup_{\Delta_{i,j}\rightarrow 0}\frac{\mathbb{E}[T_{\mathcal{A}_2}(\Delta_{i,j})]}{\Delta_{i,j}^{-2}\log\log\Delta_{i,j}^{-2}} \geq \frac{e^{-1/4}(2 - 4\delta)}{2\pi}.
	\end{align}
	
	Then, by the previous work \cite{FKLowerBound2004}, we obtain another lower bound on ranking two items, i.e., $\Omega(\Delta_{i,j}^{-2}\log\delta^{-1})$. Summing up this lower bound and (\ref{Eq:LogLogTwo-o2}), we obtain the desired lower bound. This completes the proof.
\end{proof}

\subsection{Proof of Lemma~\ref{Lemma:P12Reductions}}\label{Sec:PfReductions}

\RestateReductions*

\begin{proof}
	We first prove the reduction from $\mathcal{P}_1$ to exact ranking. Given an instance of $\mathcal{P}_1$, we simply use an exact ranking algorithm to find its true ranking. By the assumptions made in the construction of $\mathcal{P}_1$, the comparison probabilities under the correct hypothesis $\mathcal{H}_{\vec{\pi}^0}$ is exactly the same as the corresponding ranking instance. Thus, by the found true ranking, we can find the true hypothesis with no more comparisons. This completes the first part of Lemma~\ref{Lemma:P12Reductions}.
	
	Secondly, we prove the reduction from $\mathcal{P}_2$ to $\mathcal{P}_1$. Assume that $n$ is odd and $n=2m+1$, and when $n$ is even, we can prove the same results by similar steps. Let $\mathcal{B}$ be an arbitrary $\delta$-correct algorithm for $\mathcal{P}_1$. Let the ${n \choose 2}$ coins satisfying the restrictions of $\mathcal{P}_2$ be given. We construct $n$ virtual items indexed by $r_1,r_2,...,r_n$, where $(r_1,r_2,...,r_n)$ is a permutation of $[n]$. With these $n$ items, we construct $2^m$ hypotheses as defined in the construction of Problem~$\mathcal{P}_1$ (i.e., $\mathcal{H}_{\vec{\pi}},\vec{\pi}\in\{0,1\}^m$). Then, we send these $n$ items and the hypotheses as the input to algorithm $\mathcal{B}$. Whenever $\mathcal{B}$ wants a comparison over the pair $(r_i,r_j)$, we toss the coin $C_{i,j}$. If the toss gives a head, we tell $\mathcal{B}$ that the winner of the comparison is $r_i$, and if the toss gives a tail, we tell $\mathcal{B}$ that the winner is $r_j$. Since the values of the head probabilities $\mu_{i,j}$ are lawful for the comparison probabilities of Problem~$\mathcal{P}_1$, $\mathcal{B}$ does not notice any abnormal and works as usual. Finally, $\mathcal{B}$ terminates and returns a $\vec{\pi}\in\Pi$. 
	
	For any $k\in[m]$, if $\pi(k)=1$, then we return $\mu_{2k-1,2k}>1/2$, and otherwise, we return $\mu_{2k-1,2k}<1/2$. If $\mathcal{B}$ returns a correct hypothesis for these $n$ virtual items, one can determine whether $\mu_{2k-1,2k}>1/2$ for any $k\in[n]$ by no more tosses of coins. Moreover, for any $(i,j)$, the head probability $\mu_{i,j}$ of problem $\mathcal{P}_1$ equals to $p_{r_i,r_j}$, the comparison probability of problem $\mathcal{P}_2$. This completes the second part of Lemma~\ref{Lemma:P12Reductions}. The proof is complete.
\end{proof}

\subsection{Proof of Lemma~\ref{Lemma:LBP2}}\label{Sec:PfLBP2}

\RestateLBPtwo*

\begin{proof}
	In $\mathcal{P}_2$, the tosses of the coins are independent across time and coins. Also, whether one coin has head probability larger than $1/2$ is independent of other coins. Thus, $\mathcal{P}_2$ is simply a problem such that given $m$ coins with head probability not equal to $1/2$, to identify all the coins with head probabilities larger than $1/2$, and the total error probability is no more than $\delta$.
	
	Given a coin with non-$1/2$ head probability, deciding whether the head probability is larger than $1/2$ is equivalent to the problem of ranking two items, as a toss of a coin with head probability $\eta$ can be viewed as a comparison of items $i$ and $j$ with $p_{i,j} = \eta$. Thus, for coin $C_{2k-1,2k}$, to find whether $\mu_{2k-1.2k}>1/2$ with at most $\delta_k$ error probability, the expected number of tosses is at least 
	\begin{align}
	\tilde{\Omega}\big(\Delta_{q_{2k-1},q_{2k}}^{-2}\log\log\Delta_{q_{2k-1},q_{2k}}^{-1}\big) + \Omega\big(\Delta_{q_{2k-1},q_{2k}}^{-2}\log(1/\delta_k)\big).\nonumber
	\end{align}
	Here, we note that $|\mu_{{2k-1},{2k}}-1/2| = \Delta_{q_{2k-1},q_{2k}}$ for any $k\in[m]$ due to the constructions of $\mathcal{P}_1$ and $\mathcal{P}_2$.
	
	Let $\delta_k$ be the error probability incurred by determining whether $\mu_{2k-1,2k}>1/2$. To solve $\mathcal{P}_2$ with confidence $1-\delta$, it is necessary that 
	\begin{align}
	\prod_{k\in[m]}(1-\delta_k)\geq 1-\delta.\nonumber
	\end{align} 
	
	We also have that for $\delta\in(0,1/2)$,
	\begin{align}
	\sum_{k\in[m]}\delta_k \leq  & -\sum_{k\in[m]}\log(1-\delta_k) 
	=  -\log\prod_{k\in[m]}(1-\delta_k) \nonumber\\
	\leq & -\log(1-\delta) 
	=  \log(1+ \delta/(1-\delta)) \nonumber \\ \leq & \delta/(1-\delta) \leq 2\delta. \nonumber
	\end{align}
	
	Thus, the lower bound of $\mathcal{P}_2$ is at least 
	\begin{align}
	    & \tilde{\Omega}\Big( \sum_{k\in[m]}\Delta_{q_{2k-1},q_{2k}}^{-2}\log\log\Delta_{q_{2k-1},q_{2k}}^{-1}\Big) \nonumber \\
	    & +  \Omega\Big(\min\{\sum_{k\in[m]}{\Delta_{q_{2k-1},q_{2k}}^{-2}\log\delta_k^{-1}}: \sum_{k\in[m]}{\delta_k} \leq 2\delta\}\Big).\nonumber 
	\end{align}
	
	This completes the proof of Lemma~\ref{Lemma:LBP2}. 
\end{proof}

\end{appendices}

\end{document}